%% file: mainfile.tex
\documentclass[conference]{IEEEtran}
\IEEEoverridecommandlockouts

\newcommand{\myparatight}[1]{\smallskip\noindent{\bf {#1}:}~}

\usepackage[skip=0pt]{caption}

\usepackage{amsthm}
\usepackage{amsmath}
\usepackage{amssymb}
\usepackage{amsmath} 
\usepackage{amsthm}

\newtheorem{theorem}{Theorem}

\usepackage{booktabs} 
\usepackage{array} 
\usepackage{multirow}
\usepackage{graphics}
\usepackage{graphicx}

\usepackage{algorithm}
\usepackage{algpseudocode}

\usepackage{amsfonts}
\usepackage{color, url}

\usepackage{cite}

\usepackage{caption}
\usepackage{subcaption}

\usepackage{tabto}
\usepackage{xcolor, colortbl}
\usepackage{enumitem}
\usepackage{bbm}
\usepackage{dsfont}

\usepackage{bm}
\usepackage[mathscr]{eucal}

\usepackage{makecell} 
\usepackage{amssymb}  
\usepackage{pifont} 
\usepackage{xcolor} 
\usepackage{amsmath}

\algnewcommand\algorithmicforpara{\textbf{for}}
\algnewcommand\algorithmicdoinparallel{\textbf{do in parallel}}
\algdef{S}[FOR]{ForParallel}[1]{\algorithmicforpara\ #1\ \algorithmicdoinparallel}

\usepackage{xspace}

\newcommand{\alg}{\textsf{EquFL}\xspace}

\usepackage{color, colortbl}
\definecolor{greyL}{RGB}{230,248,255}
\definecolor{ffe8e7L}{RGB}{255,255,248}

\usepackage{wrapfig}

\usepackage{bigstrut,multirow,rotating}
\usepackage{placeins}
\usepackage{pgfplots}
\pgfplotsset{compat=1.17}
\makeatletter
\@namedef{r@tocindent4}{0pt}
\makeatother

\usepackage{hyperref}
\hypersetup{
	colorlinks=true,
	linkcolor=blue,
	filecolor=magenta,      
	urlcolor=cyan,
        citecolor=blue,
	pdftitle={Overleaf Example},
	pdfpagemode=FullScreen,
}

\newtheorem{assumption}{Assumption}

\newtheorem{lemma}{Lemma}

\newtheorem*{remark}{Remark}

\definecolor{PineGreen}{RGB}{0,139,114}
\definecolor{BrickRed}{RGB}{140,55,62}

\newcommand{\cmark}{{\color{PineGreen}\ding{51}}}%
\newcommand{\xmark}{{\color{BrickRed}\ding{55}}}%

\def\BibTeX{{\rm B\kern-.05em{\sc i\kern-.025em b}\kern-.08em
    T\kern-.1667em\lower.7ex\hbox{E}\kern-.125emX}}

\begin{document}

\title{When the Server Steps In: Calibrated Updates for Fair Federated Learning}

\author{
\IEEEauthorblockN{
Tianrun Yu$^\dagger$\textsuperscript{*},
Kaixiang Zhao$^\P$\textsuperscript{*},
Cheng Zhang$^\ddagger$,
Anjun Gao$^\S$,
Yueyang Quan$^{\|}$,
Zhuqing Liu$^\|$,
Minghong Fang$^\S$
}
\IEEEauthorblockA{
$^\dagger$The Pennsylvania State University,
$^\P$University of Notre Dame, 
$^\ddagger$North Carolina State University \\
$^\|$University of North Texas,
$^\S$University of Louisville
}
}

\maketitle

\makeatletter
\renewcommand{\@makefntext}[1]{\noindent#1}
\makeatother
\footnotetext{\textsuperscript{*}Equal contribution. Tianrun Yu, Kaixiang Zhao, and Cheng Zhang conducted this research while they were interns under the supervision of Zhuqing Liu and Minghong Fang.}

\begin{abstract}

Federated learning (FL) has emerged as a transformative distributed learning paradigm, enabling multiple clients to collaboratively train a global model under the coordination of a central server without sharing their raw training data. While FL offers notable advantages, it faces critical challenges in ensuring fairness across diverse demographic groups. To address these fairness concerns, various fairness-aware debiasing methods have been proposed. However, many of these approaches either require modifications to clients’ training protocols or lack flexibility in their aggregation strategies. In this work, we address these limitations by introducing \alg, a novel server-side debiasing method designed to mitigate bias in FL systems. \alg operates by allowing the server to generate a single calibrated update after receiving model updates from the clients. This calibrated update is then integrated with the aggregated client updates to produce an adjusted global model that reduces bias. Theoretically, we establish that \alg converges to the optimal global model achieved by FedAvg and effectively reduces fairness loss over training rounds. Empirically, we demonstrate that \alg significantly mitigates bias within the system, showcasing its practical effectiveness.

\end{abstract}

\begin{IEEEkeywords}
Federated learning, Fairness, Server-side Debiasing
\end{IEEEkeywords}

\input{introduction}

\input{related}

\input{method}

\input{theory}
\input{experiments}

\input{discussion}

\input{conclusion}

\bibliography{refs}
\bibliographystyle{IEEEtran}

\onecolumn
\input{appendix}

\end{document}

%% file: introduction.tex

\section{Introduction} \label{sec:intro}

Federated learning (FL)~\cite{McMahan17} is a distributed paradigm where multiple clients collaboratively train a global model under a central server without sharing raw data, thereby respecting privacy. However, despite these advantages, FL faces growing concerns regarding fairness~\cite{chang2023bias,chen2023privacy,guo2023fedbr,li2021ditto,li2019fair}.
Due to the decentralized nature of FL and the heterogeneity of client data, the resulting model may favor data-rich or majority groups while underperforming on underrepresented populations. For example, in a collaborative FL setting among banks training a loan approval model, each bank may serve distinct demographics, and the shared model could yield uneven accuracy across subgroups, disadvantaging certain communities. 
This imbalance raises fairness concerns, motivating algorithms that ensure equitable performance across diverse clients.

To date, several existing works have explored debiasing methods to address fairness risks in FL, either from the client or server side~\cite{ezzeldin2023fairfed,fan2022improving,he2020geometric,kamiran2012data,roh2020fairbatch,zeng2021improving}. Client-side approaches typically modify local training by reweighting data samples or incorporating fairness-aware regularization terms, while server-side methods adjust the global model using aggregated statistics of sensitive attributes like race or gender. Despite their promise, these methods face several limitations: they often require access to or modification of local training procedures, are narrowly tailored to specific fairness metrics, depend on FedAvg for aggregation, and lack theoretical guarantees for convergence or fairness improvement, limiting their practical applicability and generality.

 \noindent
\textbf{Our contribution: }%
To address this gap, we propose \alg, a novel server-side debiasing method for FL that is both fairness-metric agnostic and compatible with arbitrary aggregation rules. This flexibility enables seamless integration into a wide range of FL settings without modifying client procedures or restricting aggregation strategies. 
A unique feature of \alg is that it allows the server to maintain its own dataset, which it uses to proactively generate a \emph{single} calibrated update that enhances fairness in the overall system.
Ideally, the server would have a small, reliable dataset~\cite{cao2020fltrust,park2021sageflow,wang2022flare,xie2019zeno}, assuming its data shares a common distribution with that of the clients.
However, this assumption is often unrealistic in FL, as client data remains decentralized and inaccessible, leaving the server with limited insight into client distributions and hindering alignment with its own dataset.
To overcome this challenge, \alg stores early-round model checkpoints at the server and uses them to synthesize a dataset that approximates the training dynamics observed across clients.
This synthetic dataset enables the server to construct a calibrated update, which is then merged with the aggregated client updates to produce a fairer global model.

 \noindent
\textbf{Theory:}
Our first major theoretical result establishes that, under certain mild assumptions, the final global model learned by \alg will converge to the optimal global model achieved by FedAvg~\cite{McMahan17}. This convergence signifies that our proposed \alg maintains the overall accuracy of the model without compromising its performance, even with additional fairness-driven adjustments. 
Our second theoretical result rigorously proves that \alg effectively enhances fairness within the FL system. Specifically, we show that, in any given training round, the fairness loss produced by \alg is consistently lower than that produced by the standard FedAvg. This indicates that our method not only achieves comparable accuracy but also actively reduces bias.

 \noindent
\textbf{Evaluation: }%
We conduct an extensive empirical evaluation of our \alg using six datasets spanning diverse domains, seven debiasing strategies, and five aggregation rules. The results demonstrate that our approach significantly enhances fairness in FL systems. 
In addition to enhancing fairness, our proposed \alg excels in preserving the utility of the final global model, such as maintaining high accuracy, demonstrating its ability to strike a balance between fairness and performance.

Our primary contributions are summarized as follows:

\begin{list}{\labelitemi}{\leftmargin=1em \itemindent=-0.08em \itemsep=.2em}
    \item 
    We propose \alg, a novel server-side debiasing method for FL that operates independently of fairness metrics and aggregation rules, ensuring its adaptability and applicability across a wide range of FL scenarios.

    \item 
We provide theoretical guarantees showing that \alg converges to the optimal global model of FedAvg while consistently reducing fairness loss during training, ensuring both accuracy and improved fairness.

    \item 
    Extensive evaluations on multiple FL benchmarks confirm that \alg outperforms state-of-the-art fairness-aware methods in terms of both fairness improvement and accuracy retention across diverse practical settings.

\end{list}

%% file: related.tex

\section{Preliminaries and Related Work} \label{sec:related}

\subsection{Federated Learning (FL)}

We consider a FL system with a central server and $n$ clients, where each client $i$ holds a private training data $\mathcal{D}_i$, and the overall dataset is $\mathcal{D} = \cup_{i=1}^n \mathcal{D}_i$. The goal is to collaboratively learn a global model $\mathbf{w} \in \mathbb{R}^d$ by minimizing the objective $\mathcal{L}(\mathbf{w}) = \sum_{i=1}^{n} \alpha_i \mathcal{L}_i(\mathbf{w}, \mathcal{D}_i)$, where $\alpha_i = \frac{|\mathcal{D}_i|}{|\mathcal{D}|}$ and $\mathcal{L}_i$ is the local loss of client $i$. In each training round, the server first broadcasts the current global model $\mathbf{w}^t$ to all or a subset of clients. Each participating client then samples a mini-batch $\mathcal{B}_i^t \subset \mathcal{D}_i$ and computes a local update $\mathbf{g}_i^t = \frac{1}{|\mathcal{B}_i^t|} \sum_{z \in \mathcal{B}_i^t} \nabla \mathcal{L}_i(\mathbf{w}, z)$, which is sent back to the server. The server aggregates the collected updates using an aggregation rule $\mathsf{GAR}(\cdot)$ and updates the global model as $\mathbf{w}^{t+1} = \mathbf{w}^t - \eta_t \cdot \mathsf{GAR}(\mathbf{g}^t_1, \dots, \mathbf{g}^t_n)$, where $\eta_t$ is the learning rate. For example, FedAvg~\cite{McMahan17} uses a weighted average: $\mathsf{GAR} = \sum_{i=1}^n \frac{|\mathcal{D}_i|}{|\mathcal{D}|} \mathbf{g}^t_i$.

\subsection{Fairness Metrics}
\label{fairness_metrics}

Fairness~\cite{dwork2012fairness,fang2022fairroad} is typically measured in terms of group fairness and individual fairness.
Group fairness ensures similar treatment across demographic groups, while individual fairness requires similar outcomes for similar individuals. This paper evaluates four metrics in FL: equalized odds~\cite{hardt2016equality}, demographic parity~\cite{dwork2012fairness}, and calibration~\cite{pleiss2017fairness} for group fairness, and consistency~\cite{zemel2013learning} for individual fairness. 
Each metric yields a bias score, where lower values reflect greater fairness.

\myparatight{1)~Equalized odds~\cite{hardt2016equality}}%
To define equalized odds, assume each data point has a sensitive attribute $A$ (e.g., race or gender), dividing the data into groups $G$, and a true label $Y \in \{0,1\}$, with $\hat{Y}(\mathbf{w})$ as the model’s prediction. The metric $\mathcal{M}_{\text{EO}}(\mathbf{w}, \mathcal{D})$ measures the maximum difference in prediction rates between any two groups with the same true label:
\begin{align}
\label{eo_equ}
  \mathcal{M}_{\text{EO}}(\mathbf{w},\mathcal{D})=&  \max_{y \in Y, h, k \in G} |\underset{\mathcal{D}}{\mathbb{P}}(\hat{Y}(\mathbf{w}) = 1 | A = h, Y = y) 
  \notag\\&- \underset{\mathcal{D}}{\mathbb{P}}(\hat{Y}(\mathbf{w}) = 1 | A = k, Y = y)|,
\end{align}
where $\mathbb{P}_{\mathcal{D}}(\hat{Y} = 1 \mid A = h, Y = y)$ is the probability the model predicts label 1 for group $h$ given true label $y$, evaluated over dataset $\mathcal{D}$.

\myparatight{2)~Demographic parity~\cite{dwork2012fairness}}%
Demographic parity assesses whether a model gives equal positive prediction rates across groups defined by a sensitive attribute $A$, aiming to prevent systematic favoritism.
The metric $\mathcal{M}_{\text{DP}}(\mathbf{w}, \mathcal{D})$ is defined as:
\begin{align}
\mathcal{M}_{\text{DP}}(\mathbf{w},\mathcal{D}) = & \max_{h, k \in G} |\underset{\mathcal{D}}{\mathbb{P}}(\hat{Y}(\mathbf{w}) = 1 | A = h)  
 \notag\\& - \underset{\mathcal{D}}{\mathbb{P}}(\hat{Y}(\mathbf{w}) = 1 | A = k)|,
\end{align}
where $\mathbb{P}_{\mathcal{D}}(\hat{Y} = 1 \mid A = h)$ is the probability the model assigns a positive label to group $h$ in $\mathcal{D}$.

\myparatight{3)~Calibration~\cite{pleiss2017fairness}}%
The calibration metric measures how well predicted probabilities align with actual outcomes across groups defined by a sensitive attribute $A$. A lower value indicates that positive predictions are equally reliable across all groups. It is defined as:
\begin{align}
\mathcal{M}_{\text{CAL}}(\mathbf{w},\mathcal{D}) = & \max_{h \in G}|\underset{\mathcal{D}}{\mathbb{P}}(Y=1 | \hat{Y}(\mathbf{w}) = 1 ,A=h)
      \notag\\&  -  \underset{\mathcal{D}}{\mathbb{P}}(Y=1 | \hat{Y}(\mathbf{w}) = 1 )|,
\end{align}
where $\mathbb{P}_{\mathcal{D}}(Y = 1 \mid \hat{Y} = 1, A = h)$ is the probability of a true positive in group $h$, and $\mathbb{P}_{\mathcal{D}}(Y = 1 \mid \hat{Y} = 1)$ is the overall true positive rate. A value of zero indicates perfect calibration.

\myparatight{4)~Consistency~\cite{zemel2013learning}}%
Consistency is an individual fairness metric that assesses whether a model gives similar predictions to similar inputs. It is defined as:
\begin{align}
\mathcal{M}_{\text{CON}}(\mathbf{w},\mathcal{D}) =  \frac{1}{|\mathcal{D}|} \sum_{z \in \mathcal{D}} |\hat{Y}_z(\mathbf{w}) - \frac{1}{|\mathcal{N}_z|} \sum_{q \in \mathcal{N}_z} \hat{Y}_q(\mathbf{w})|,
\end{align}
where $\hat{Y}_z(\mathbf{w})$ is the prediction for sample $z$, and $\mathcal{N}_z$ is its nearest neighbors.

\begin{table}
    \footnotesize
	\centering
    	\caption{Comparison of debiasing methods. ``Aggregation rule independent'' means the method works with any aggregation rule. ``Fairness metric agnostic'' indicates compatibility with various fairness definitions. ``Theoretical guarantee'' means the method is backed by theoretical analysis.
	}
	\label{tab:comparison}
	\begin{tabular}{|c|c|c|c|}
		\hline
		Method & \multicolumn{1}{c|}{\makecell{Aggregation rule\\independent}}  
		& \multicolumn{1}{c|}{\makecell{Fairness metric\\agnostic}}      
		& \multicolumn{1}{c|}{\makecell{Theoretical\\guarantee}}  \\
		\hline
		FLinear~\cite{he2020geometric} & \xmark & \xmark & \xmark \\
		\hline
		FairFed~\cite{ezzeldin2023fairfed}  & \xmark & \xmark & \xmark \\
		\hline
		FedFB~\cite{zeng2021improving}  & \xmark & \cmark  & \cmark\\
		\hline
		Reweight~\cite{kamiran2012data} & \xmark & \xmark & \xmark \\
		\hline
		\alg  & \cmark & \cmark  & \cmark \\
		\hline
	\end{tabular}
\vspace{-.07in}
\end{table}

\subsection{Bias Mitigation in FL}

Fairness in FL has attracted growing interest in recent years~\cite{chen2023privacy,guo2023fedbr,li2021ditto,li2019fair,mohri2019agnostic,shi2023towards,xu2023bias,zhang2021unified,zhang2024eliminating}, leading to various methods aimed at promoting fair model outcomes in this decentralized setting~\cite{ezzeldin2023fairfed,zeng2021improving,kamiran2012data,he2020geometric,roh2020fairbatch}. Debiasing can be applied locally at the client level, where updates are adjusted before being sent to the server, or globally at the server, which leverages client-side statistics to refine the global model. However, many of these approaches face key limitations. Some, like FLinear~\cite{he2020geometric} and FairFed~\cite{ezzeldin2023fairfed}, are tailored to specific fairness metrics, limiting their adaptability. Others assume the use of simple aggregation rules such as FedAvg, reducing their effectiveness under more general settings. Additionally, most lack theoretical guarantees and rely solely on empirical validation. Table~\ref{tab:comparison} highlights how our method \alg addresses these limitations to enhance fairness more broadly in FL systems.

%% file: method.tex

\section{Our Method} 

\subsection{Overview}

We propose \alg, a server-side debiasing method for FL that is both effective and efficient. It reduces bias across diverse client distributions without compromising accuracy or introducing significant overhead. \alg is compatible with various fairness metrics and aggregation rules, and does not require additional client information beyond what is used in FedAvg. 
The server leverages early global models to create a synthetic dataset that reflects client training behavior. Using this dataset, it generates a calibrated update that is combined with incoming client updates, resulting in a global model with improved fairness.
Remark that this work focuses on fairness in a non-adversarial FL setting with honest clients and clean data, excluding attacks or poisoned updates~\cite{fang2020local,shejwalkar2022back,fang2022aflguard,fang2024byzantine,fang2025we,wang2025poisoning,fang2025byzantine,zhang2024poisoning,fang2025provably}.

\subsection{Generation of Synthetic Data}
\label{section_Synthetic_data}

\alg relies on generating a synthetic dataset on the server to produce a calibrated update that improves fairness. Rather than assuming access to a representative server-side dataset, as done in prior work~\cite{cao2020fltrust,wang2022flare,xie2019zeno}, or requiring clients to share their private data, we leverage recent advances in dataset condensation~\cite{cazenavette2022dataset,kim2022dataset} to construct a synthetic dataset that approximates the learning dynamics of the clients’ data.

Specifically, assume that the server saves the global models from the first $s$ rounds, which we represent as $\{\mathbf{w}^1, \mathbf{w}^2, \ldots, \mathbf{w}^s\}$. In FL, these global model checkpoints capture cumulative knowledge gained from training over multiple rounds across distributed clients. The primary goal for the server is to leverage these collected global model checkpoints $\{\mathbf{w}^1, \mathbf{w}^2, \ldots, \mathbf{w}^s\}$ to construct a synthetic dataset.
This synthetic dataset, denoted as $\mathcal{D}_{\text{syn}} = \{\mathbf{X}_{\text{syn}}, \mathbf{Y}_{\text{syn}}\}$, comprises synthetic inputs $\mathbf{X}_{\text{syn}}$ and their corresponding labels $\mathbf{Y}_{\text{syn}}$.
This synthetic dataset should enable the neural network $f$, when trained on $\mathcal{D}_{\text{syn}}$, to achieve a performance comparable to training on the clients' overall training dataset $\mathcal{D}$, which aggregates data from all clients. To achieve this, it is crucial that $\mathcal{D}_{\text{syn}}$ preserves the statistical properties and essential knowledge from the FL training process.
To achieve this, we start by randomly selecting two model check-points from the trajectory: $\mathbf{w}^{\tau}$ and $\mathbf{w}^{\tau+\vartheta}$, i.e., $\mathbf{w}^{\tau}, \mathbf{w}^{\tau+\vartheta} \in \{\mathbf{w}^1, \mathbf{w}^2, \ldots, \mathbf{w}^s\}$, where $1 \le \tau < s$, and $\vartheta>0$.
The idea is to use $\mathcal{D}_{\text{syn}}$ to train the model from checkpoint $\mathbf{w}^{\tau}$ for $\vartheta$ steps, resulting in a model state that closely matches $\mathbf{w}^{\tau+\vartheta}$. In essence, the synthetic dataset should replicate the learning dynamics that would have been produced if training were performed on $\mathcal{D}$, allowing the model to transition smoothly between these states.

This synthetic data generation objective can be formulated as an optimization problem, where the aim is to minimize the difference between the model trained on the synthetic data and the target model checkpoint $\mathbf{w}^{\tau+\vartheta}$:
\begin{equation}
\begin{split}
\min_{\mathbf{X}_{\text{syn}}, \mathbf{Y}_{\text{syn}}} \Pi(\mathbf{X}_{\text{syn}}, \mathbf{Y}_{\text{syn}}) = || \overrightarrow{\mathbf{w}} -\mathbf{w}^{\tau+\vartheta}||^2,
\\
\text{s.t.} \quad \overrightarrow{\mathbf{w}}= f(\mathbf{X}_{\text{syn}}, \mathbf{Y}_{\text{syn}}, \mathbf{w}^{\tau}, \vartheta),
\end{split}
\end{equation}
where $f(\mathbf{X}_{\text{syn}}, \mathbf{Y}_{\text{syn}}, \mathbf{w}^{\tau}, \vartheta)$ denotes the updated model parameters, represented as $\overrightarrow{\mathbf{w}}$, obtained by training the neural network $f$ on the current synthetic dataset for $\vartheta$ iterations, beginning with the model $\mathbf{w}^{\tau}$. 
The objective is to determine the synthetic features $\mathbf{X}_{\text{syn}}$ and labels $\mathbf{Y}_{\text{syn}}$ such that the resulting model $\overrightarrow{\mathbf{w}}$ is as close as possible to the target $\mathbf{w}^{\tau+\vartheta}$.

To solve the above optimization problem, we use an iterative gradient descent approach, as shown in Algorithm~\ref{syn_data_gen} in Appendix. 
During each iteration, two checkpoints, $\mathbf{w}^{\tau}$ and $\mathbf{w}^{\tau+\vartheta}$, are randomly selected from the set of collected checkpoints $\{\mathbf{w}^1, \mathbf{w}^2, \ldots, \mathbf{w}^s\}$. The neural network $f$ is then trained for $\vartheta$ steps on the current synthetic dataset, starting from the checkpoint $\mathbf{w}^{\tau}$. Next, the resulting model state $\overrightarrow{\mathbf{w}}$ is evaluated against the target checkpoint $\mathbf{w}^{\tau+\vartheta}$. We calculate the gradient of the loss function $\Pi(\mathbf{X}_{\text{syn}}, \mathbf{Y}_{\text{syn}})$ with respect to both $\mathbf{X}_{\text{syn}}$ and $\mathbf{Y}_{\text{syn}}$, denoted as $\nabla_{\mathbf{X}_{\text{syn}}} \Pi(\mathbf{X}_{\text{syn}}, \mathbf{Y}_{\text{syn}})$ and $\nabla_{\mathbf{Y}_{\text{syn}}} \Pi(\mathbf{X}_{\text{syn}}, \mathbf{Y}_{\text{syn}})$. These gradients are then used to update the synthetic dataset through gradient descent, refer to Line~\ref{syn_data_update} in Algorithm~\ref{syn_data_gen}.
By iteratively optimizing this process, we generate a synthetic dataset that captures the essential learning trajectory of the global model. 
This dataset acts as a highly effective substitute for the clients’ collective training data, allowing the server to avoid requesting clients to share their local training data or relying on unrealistic assumptions, such as the server possessing a dataset that mirrors the distribution of all client data accurately.
It is important to note that sensitive attributes are included within the sample features, so there is no need to generate them separately for the synthetic dataset.

\subsection{Generation of Calibrated Update}

With the synthetic dataset $\mathcal{D}_{\text{syn}}$ available, the server generates a calibrated update $\mathbf{g}_0^t$ in each round $t$ to mitigate bias. After collecting client updates $\mathbf{g}_1^t, \dots, \mathbf{g}_n^t$, the server adjusts the global model as:
\begin{align}
\label{compute_w_t_1}
\mathbf{w}^{t+1} = \mathbf{w}^t - \eta_t \cdot (\gamma_t \cdot \mathbf{g}_0^t + \mathsf{GAR}(\mathbf{g}_1^t, \dots, \mathbf{g}_n^t)),
\end{align}
where $\mathsf{GAR}(\cdot)$ is the aggregation rule and $\gamma_t > 0$ balances the calibrated update.
Since the server cannot access clients' local data, we propose learning $\mathbf{g}_0^t$ by optimizing a fairness metric $\mathcal{M}$ (e.g., equalized odds or demographic parity) over the synthetic dataset $\mathcal{D}_{\text{syn}}$. 
The calibrated update generation (CUG) problem is formulated as:
\begin{align}
\label{CUG_problem_upper}
\!\!\!\!& \min_{\mathbf{g}^t_{0}} \mathcal{F}(\mathbf{w}^{t+1},\mathcal{D}_{\text{syn}}) = \mathcal{M}( \mathbf{w}^{t+1},\mathcal{D}_{\text{syn}}), \quad
 \notag\\& \text{s.t. } \mathbf{w}^{t+1} = \mathbf{w}^t - \eta_t (\gamma_t \cdot \mathbf{g}_0^t + \mathsf{GAR}(\mathbf{g}_1^t, \dots, \mathbf{g}_n^t))).
\end{align}
Problem CUG offers a flexible framework applicable to various fairness metrics. For example, to enforce equalized odds, we instantiate $\mathcal{M}$ with $\mathcal{M}_{\text{EO}}$, as defined in Eq.~(\ref{eo_equ}).
However,
optimizing 
Problem CUG is challenging due to the non-differentiability of fairness metrics (e.g., threshold-based classification) and potentially complex $\mathsf{GAR}(\cdot)$. 
In the following, we demonstrate how to address these challenges
using the equalized odds fairness metric as an example.

Equalized odds measures whether a model maintains balanced accuracy across groups by evaluating the maximum gap in true and false positive rates. Following insights from~\cite{shen2022optimising}, minimizing the loss difference for positive predictions between groups approximates this goal. Thus, we reformulate Eq.~(\ref{CUG_problem_upper}) for the equalized odds metric as:
\begin{align}
\label{CUG_problem_upper_tran}
      \min_{\mathbf{g}^t_{0}} \mathcal{F}_{\text{EO}}(\mathbf{w}^{t+1},\mathcal{D}_{\text{syn}})  \!=&\! \sum_{y \in Y} \sum_{h, k \in G } \bigg| \frac{1}{|\mathcal{D}_{\text{syn}}^{h,y}|}  
    \sum_{z \in \mathcal{D}_{\text{syn}}^{h,y}}  l(\mathbf{w}^{t+1}, z) 
     \notag\\& - 
    \frac{1}{|\mathcal{D}_{\text{syn}}^{k,y}|} \sum_{q \in \mathcal{D}_{\text{syn}}^{k,y}} \! l(\mathbf{w}^{t+1}, q) \bigg|,
\end{align}
where $\mathcal{D}_{\text{syn}}^{h,y}$ denotes the subset of synthetic data with group $h$ and label $y$. The loss $l(\mathbf{w}, z)$ measures the discrepancy between prediction and true label. For cross-entropy loss, $l(\mathbf{w}, z) = -y_z \log p(\mathbf{w}, z) - (1 - y_z) \log(1 - p(\mathbf{w}, z))$, where $p(\mathbf{w}, z) = \frac{1}{1 + e^{-\mathbf{w}^\top \mathbf{x}_z}}$ is the predicted probability for the positive class.
Although we reformulate Eq.~(\ref{CUG_problem_upper}) as Eq.~(\ref{CUG_problem_upper_tran}), computing the gradient of $\mathcal{F}_{\text{EO}}(\mathbf{w}^{t+1}, \mathcal{D}_{\text{syn}})$ to derive the calibrated update $\mathbf{g}^t_{0}$ remains difficult, as $\mathbf{w}^{t+1}$ depends on the aggregation rule $\mathsf{GAR}(\cdot)$, which is generally non-differentiable.
In what follows, we detail our approach to solve Eq.~(\ref{CUG_problem_upper_tran}) and efficiently compute \( \mathbf{g}^t_{0} \).

We denote $\overline{\mathbf{g}}^t = \gamma_t \cdot \mathbf{g}_0^t + \mathsf{GAR}(\mathbf{g}^t_{1}, \mathbf{g}^t_{2}, \cdots, \mathbf{g}^t_{n})$.
We approximate the left-hand side of Eq.~(\ref{CUG_problem_upper_tran}) as:
\begin{align}
 &\min_{\mathbf{g}^t_{0}} \mathcal{F}_{\text{EO}}(\mathbf{w}^{t+1},\mathcal{D}_{\text{syn}})   \stackrel{(a)}  = 
 \min_{\mathbf{g}^t_{0}} \mathcal{F}_{\text{EO}}(\mathbf{w}^t - \eta_t \cdot \overline{\mathbf{g}}^t,  \mathcal{D}_{\text{syn}})  \nonumber
 \\
&  \stackrel{(b)}\approx
 \min_{\mathbf{g}^t_{0}} \,  \mathcal{F}_{\text{EO}}(\mathbf{w}^t, \mathcal{D}_{\text{syn}}) 
  - \eta_t \cdot \nabla_{\mathbf{w}^t} \mathcal{F}_{\text{EO}}(\mathbf{w}^t, \mathcal{D}_{\text{syn}})^\top \overline{\mathbf{g}}^t,
  \label{appro_second}
\end{align}
where $(a)$ results from substituting $\mathbf{w}^{t+1} = \mathbf{w}^t - \eta_t \overline{\mathbf{g}}^t$, and \((b)\) is derived using the Taylor expansion that $
\mathcal{F}_{\text{EO}}(\mathbf{w}^t - \eta_t \cdot \overline{\mathbf{g}}^t, \mathcal{D}_{\text{syn}}) 
\approx \mathcal{F}_{\text{EO}}(\mathbf{w}^t, \mathcal{D}_{\text{syn}}) - \eta_t \cdot \nabla_{\mathbf{w}^t} \mathcal{F}_{\text{EO}}(\mathbf{w}^t, \mathcal{D}_{\text{syn}})^\top \overline{\mathbf{g}}^t + \mathcal{O}(\eta_t^2 \|\overline{\mathbf{g}}^t\|^2)
$, omitting higher-order terms.
In Eq.~(\ref{appro_second}), $\mathbf{w}^t$ and $\eta_t$ are fixed at round $t$, making $\mathcal{F}_{\text{EO}}(\mathbf{w}^t, \mathcal{D}_{\text{syn}})$ and $\eta_t$ constant. 
Thus, we can omit them and reformulate Eq.~(\ref{appro_second}) as an equivalent maximization problem:
\begin{align}
\label{max_opt}
 \max_{\mathbf{g}^t_{0}} \nabla_{\mathbf{w}^t}
\mathcal{F}_{\text{EO}}(\mathbf{w}^t, \mathcal{D}_{\text{syn}} )^\top \overline{\mathbf{g}}^t.
\end{align}
By definition of $\overline{\mathbf{g}}^t$, we have
$
\nabla_{\mathbf{w}^t} \mathcal{F}_{\text{EO}}(\mathbf{w}^t, \mathcal{D}_{\text{syn}} )^\top \overline{\mathbf{g}}^t = \nabla_{\mathbf{w}^t} \mathcal{F}_{\text{EO}}(\mathbf{w}^t, \mathcal{D}_{\text{syn}} )^\top (\gamma_t \cdot \mathbf{g}^t_{0} + \mathsf{GAR}(\mathbf{g}^t_{1}, \mathbf{g}^t_{2}, \cdots, \mathbf{g}^t_{n})).
$
Treating $\gamma_t$ and $\mathsf{GAR}(\mathbf{g}^t_{1}, \mathbf{g}^t_{2}, \cdots, \mathbf{g}^t_{n})$ as constants, Eq.~(\ref{max_opt}) reduces to:
\begin{align}
\label{max_opt_tran}
 \max_{\mathbf{g}^t_{0}} \nabla_{\mathbf{w}^t}
\mathcal{F}_{\text{EO}}(\mathbf{w}^t, \mathcal{D}_{\text{syn}} )^\top \mathbf{g}^t_{0}.
\end{align}
To maximize \(\nabla_{\mathbf{w}^t} \mathcal{F}_{\text{EO}}(\mathbf{w}^t, \mathcal{D}_{\text{syn}} )^\top \mathbf{g}^t_{0}\), we should align \(\mathbf{g}^t_{0}\) in the same direction as \(\nabla_{\mathbf{w}^t} \mathcal{F}_{\text{EO}}(\mathbf{w}^t, \mathcal{D}_{\text{syn}})\) to maximize their dot product. 
To further simplify our approach, we set $\mathbf{g}_0^t$ equal in magnitude to $\nabla_{\mathbf{w}^t} \mathcal{F}_{\text{EO}}(\mathbf{w}^t, \mathcal{D}_{\text{syn}})$, yielding the optimal choice:
\begin{align}
\label{compute_g_0}
\mathbf{g}^t_{0} =  \nabla_{\mathbf{w}^t} \mathcal{F}_{\text{EO}}(\mathbf{w}^t, \mathcal{D}_{\text{syn}}).
\end{align}
We can use autograd in PyTorch~\cite{paszke2019pytorch} or TensorFlow~\cite{abadi2016tensorflow} to compute $\nabla_{\mathbf{w}^t} \mathcal{F}_{\text{EO}}(\mathbf{w}^t, \mathcal{D}_{\text{syn}})$. The server then adds the resulting calibrated update $\mathbf{g}^t_{0}$ to the aggregated client updates to reduce bias in the global model.
This process applies to equalized odds; similar formulations for three other fairness metrics are provided in Appendix~\ref{sec:appendix_1}.
Algorithm~\ref{our_alg} in Appendix summarizes our method. In the first $s$ rounds, the server collects global models without calibration. Once $\mathcal{S} = \{\mathbf{w}^1, \ldots, \mathbf{w}^s\}$ is collected, it constructs a synthetic dataset (Lines~\ref{Construct_one}-\ref{Construct_two}). Calibrated updates are generated in all subsequent rounds, using the same synthetic dataset built at round $s+1$.

%% file: theory.tex

\section{Theoretical Performance Analysis} 
\label{sec:theoretical}

This section presents the theoretical guarantees of our method. Recall that the practical procedure runs for an initial $(t=S)$ rounds in which the server aggregates the
plain client gradients; at round $(t=S\!+\!1)$ it synthesises a proxy dataset and thereafter augments every update
with a fairness–corrective gradient computed on this synthetic set. 
To simplify the analysis, we follow a common assumption that the server holds a separate dataset prior to training. This dataset need not follow the same distribution as the clients' training data and may be out-of-distribution, as long as it satisfies Assumption~\ref{assumption_4}.
Let $\mathbf{w}^*$ be the minimizer of the global loss $\mathcal{L}$, with optimal value $\mathcal{L}^* = \mathcal{L}(\mathbf{w}^*)$. Define $\mathcal{F}$ and $\mathcal{F}_{\text{syn}}$ as the fairness losses on the clients’ training data and synthetic data, respectively. Let $\mathcal{L}_i^*$ denote the optimal loss for client $i$, and $\mathcal{F}_{\text{syn}}^*$ be the minimal fairness loss over the synthetic dataset. We define two key heterogeneity terms: $\Gamma_1 = \mathcal{L}^* - \sum_{i=1}^n \alpha_i \mathcal{L}_i^*$, and $\Gamma_2 = \mathcal{F}_{\text{syn}}(\mathbf{w}^*) - \mathcal{F}_{\text{syn}}^*$. Let $T$ be the number of rounds where the calibrated update is applied (not the total training rounds), and define $\theta = \|\mathbf{w}^1 - \mathbf{w}^*\|^2$, where $\mathbf{w}^1$ is the global model after the first calibrated update. Before stating our theoretical results, we outline the standard assumptions adopted in prior work~\cite{ChenPOMACS17,chu2022securing,karimireddy2020byzantine,li2019convergence,yin2018byzantine}.

\begin{assumption}
\label{assumption_1}
The loss functions are \(\mu\)-strongly convex and \(\rho\)-smooth.  
See Appendix~\ref{app_assumption_1} for details.
\end{assumption}

\begin{assumption}
\label{assumption_3}
The gradient of the global loss is bounded.
\begin{gather*}
\left\|\nabla \mathcal{L}(\mathbf{w}) \right\|^2 \leq R.
\end{gather*}
\end{assumption}

\begin{assumption}
\label{assumption_4}
The difference between the gradients of the synthetic fairness loss function \(\mathcal{F}_{\text{syn}}\) and the actual fairness loss function \(\mathcal{F}\) is bounded by a small constant \(\epsilon\).
\begin{gather*}
\| \nabla \mathcal{F}_{\text{syn}}(\mathbf{w}) - \nabla \mathcal{F}(\mathbf{w}) \| < \epsilon.
\end{gather*}
\end{assumption}

\begin{theorem}
Assume that Assumptions~\ref{assumption_1}-\ref{assumption_3} hold, with \(\rho\), \(\mu\), \(\nu\), and \(\theta\) defined accordingly. Suppose the server combines clients' model updates using the FedAvg rule. Set the learning rate as \(\eta_t = \frac{\varpi}{t + \varsigma}\) and \(\gamma_t = \frac{1}{t + \varsigma}\), where \(\varsigma\) and \(\varpi\) are constants and \(\varpi > \frac{1}{\mu}\). Under these conditions, our proposed \alg guarantees the following for any fairness metric:
    \begin{align}
        \mathcal{L}(\mathbf{w}^{T}) - \mathcal{L}^{*} \leq \frac{\rho}{2}\frac{ \nu}{T+\varsigma}, \nonumber
    \end{align}
    where \( \nu = \max\{\mathcal{Z}_1,\mathcal{Z}_2 \}\) with \(\mathcal{Z}_1=\theta(\varsigma+1)\) and \(\mathcal{Z}_2=\frac{4\rho\Gamma_1\varpi^{2}+2\Gamma_2\varpi}{\mu \varpi-1} \).\label{theorem1}
 
\end{theorem}
\begin{proof}
The proof is relegated to Appendix~\ref{app_proof_theorem1}.
\end{proof}

\begin{theorem}
    Assume that Assumptions~\ref{assumption_1}-\ref{assumption_4} hold. Let the server use the FedAvg rule to combine clients' model updates. Suppose there exists a constant \( \psi > \epsilon \) such that \( \|\nabla \mathcal{F}_{\text{syn}}(\mathbf{w}^t)\| \geq \psi \). Set the learning rates as \( \eta_t = \frac{\varpi}{t + \varsigma} \) and \( \gamma_t = \frac{1}{t + \varsigma} \), where \(\varsigma\) and \(\varpi\) are constants satisfying \( \varpi > \frac{1}{\mu} \) and \( \varsigma > \max \left\{ \sqrt{ \frac{\rho \varpi}{2} }, \ \frac{\rho \varpi \sqrt{R} + \sqrt{ (\rho \varpi)^2 R + 2 (\psi - \epsilon) \psi \rho \varpi}}{2 (\psi - \epsilon)} \right\} \). Under these conditions, our proposed \alg ensures the following result for any fairness metric:
    \begin{align}
        \mathcal{F}(\mathbf{w}^{t+1}) < \mathcal{F}(\mathbf{v}^{t+1}),  \nonumber
    \end{align}
where \( \mathbf{v}^{t+1} = \mathbf{w}^t - \eta_t \cdot \mathsf{GAR}(\mathbf{g}^t_{1}, \mathbf{g}^t_{2}, \dots, \mathbf{g}^t_{n}) \), with \(\mathsf{GAR}(\cdot)\) implemented here using the FedAvg rule.
    \label{theorem2}
\end{theorem}

\begin{proof}
The proof is relegated to Appendix~\ref{app_proof_theorem2}.
\end{proof}

\begin{remark} 

Theorem~\ref{theorem1} establishes that our \alg approach converges to the optimal global model, indicating that \alg preserves the model’s accuracy without any reduction in performance. Furthermore, Theorem~\ref{theorem2} highlights that EquFL achieves a lower fairness loss compared to the standard FedAvg method, demonstrating its effectiveness in improving fairness metrics.
Our theoretical analysis is based on simplifying assumptions that are widely accepted in the FL community. Nonetheless, we recognize that these assumptions may not entirely reflect real-world complexities. Extensive experimental results demonstrate that our \alg remains effective, even when some of these assumptions are only partially met, highlighting its practical applicability.

\end{remark}

%% file: experiments.tex

\section{Experiments} \label{sec:exp}

\subsection{Experimental Setup}
\subsubsection{Datasets} 

We evaluate our method on six datasets spanning structured and image data: Income-Sex~\cite{liu2015deep}, Employment-Sex~\cite{liu2015deep}, Health-Sex~\cite{liu2015deep}, Income-Race~\cite{liu2015deep}, MNIST~\cite{lecun2010mnist}, and CIFAR-10~\cite{krizhevsky2009learning}. The first four datasets are derived from US Census data and partitioned geographically into 51 parties representing 50 states and Puerto Rico. For these datasets, the sensitive attribute is sex (or race in Income-Race). For MNIST and CIFAR-10 datasets, following~\cite{zhang2024lr,zhang2025uncertainty}, we define label parity (odd vs. even digits or classes) as the sensitive attribute, given the absence of inherent sensitive features.

\subsection{Comparison Methods, Non-IID Setting, Evaluation Metric, and Parameter Settings}
We evaluate \alg against six debiasing baselines: FLinear~\cite{he2020geometric}, FairFed~\cite{ezzeldin2023fairfed}, FedFB~\cite{zeng2021improving}, Reweight~\cite{kamiran2012data}, Gaussian, and Uniform. Details of these methods are in Appendix~\ref{sec:appendix_baseline}.
Our evaluation considers the inherent Non-IID nature of FL. The census-based datasets (Income-Sex, Employment-Sex, Health-Sex, Income-Race) are naturally heterogeneous. For MNIST, we assign each client one label; for CIFAR-10, two labels per client to simulate challenging Non-IID settings.
We use four fairness metrics: equalized odds (EO), demographic parity (DP), calibration (CAL), and consistency (CON), defined in Section~\ref{fairness_metrics}. Lower bias scores indicate fairer models.
The server collects updates during the first half of training, builds a synthetic dataset with 1,000 samples using a StandardMLP (Appendix~\ref{sec:appendix_architectures}), and then begins injecting calibrated updates. The model (e.g., network $f$) used by the server to generate the synthetic data differs from those used by the clients. FedAvg is employed for aggregation. 
Parameter settings such as network architecture, learning rate, batch size, and total training rounds are provided in Appendix~\ref{sec:appendix_settings}. The Income-Sex, Employment-Sex, Health-Sex, and Income-Race datasets involve 51 clients, representing 50 states and Puerto Rico, while MNIST and CIFAR-10 are distributed across 10 clients each. The parameter $\gamma$ is set to 1 for all datasets in our \alg.
All experiments were carried out on four NVIDIA A10 GPUs.
By default, results are reported on the Income-Sex dataset and averaged over five runs.
Variance was minimal and thus excluded.

\begin{table*}[t]
\centering
 \scriptsize
\addtolength{\tabcolsep}{-2.045pt}
\caption{Results of various debiasing methods evaluated on different fairness metrics. For FedAvg, the results are represented solely as the ``bias score'',  whereas for the debiasing methods, the results are reported in the format ``bias score (fairness improvement)''.}
 \subfloat[Income-Sex.]
 {
\begin{tabular}{|l|c|c|c|c|}
\hline
Method & EO & DP & CAL & CON \\ \hline
FedAvg & 0.0611 & 0.0934 & 0.0343 & 0.1281 \\ \hline
FLinear  & 0.0428 (29.9\%) & 0.0669 (28.3\%) &0.0270 (21.2\%)  & 0.1080 (15.6\%) \\
FairFed  & 0.0490 (19.8\%) & 0.0840 (10.1\%) &0.0250 (27.1\%)  & 0.1120 (12.5\%)  \\
FedFB  & 0.0410 (32.8\%) & 0.0590 (36.8\%) &0.0310 (9.6\%)  & 0.1240 (3.2\%) \\
Reweight  & 0.0480 (21.4\%) & 0.0790 (15.4\%) &0.0310 (9.6\%)  & 0.1230 (3.9\%) \\
Gaussian  & 0.0664 (-8.6\%) & 0.0807 (13.5\%) &0.0536 (-56.2\%)  & 0.0993 (22.4\%)  \\
Uniform  & 0.0446 (27.0\%) & 0.0632 (32.3\%) &0.0296 (13.7\%)  & 0.1228 (4.1\%)  \\
\rowcolor{greyL}
\alg & 0.0335 (45.1\%) & 0.0266 (71.5\%) &0.0224 (34.6\%)  & 0.0948 (25.9\%) \\ \hline
\end{tabular}
}
 \subfloat[Employment-Sex.]
{
\begin{tabular}{|l|c|c|c|c|}
\hline
Method & EO & DP & CAL & CON \\ \hline
FedAvg & 0.0264 & 0.0108 & 0.0097 & 0.0498 \\ \hline
FLinear & 0.0216 (18.4\%) & 0.0089 (17.9\%) & 0.0124 (-27.7\%) & 0.0467 (6.3\%) \\
FairFed & 0.0252 (4.8\%) & 0.0241 (-122.1\%) & 0.0095 (2.1\%) & 0.0515 (-3.2\%) \\
FedFB & 0.0223 (15.8\%) & 0.0087 (18.9\%) & 0.0096 (1.1\%) & 0.0501 (-0.4\%) \\
Reweight & 0.0241 (9.0\%) & 0.0092 (15.2\%) & 0.0097 (0.0\%) & 0.0472 (5.3\%) \\
Gaussian & 0.0328 (-23.8\%) & 0.0092 (15.1\%) & 0.0098 (-1.1\%) & 0.0521 (-4.4\%) \\
Uniform & 0.0274 (-3.4\%) & 0.0121 (-11.8\%) & 0.0103 (-6.9\%) & 0.0511 (-2.6\%) \\
\rowcolor{greyL}
\alg & 0.0205 (22.6\%) & 0.0085 (21.6\%) & 0.0094 (3.1\%) & 0.0431 (13.5\%) \\ \hline
\end{tabular}
}
\vspace{0.08in}
\\
 \subfloat[Health-Sex.]
{
\begin{tabular}{|l|c|c|c|c|}
\hline
Method & EO & DP & CAL & CON \\ \hline
FedAvg & 0.0561 & 0.0357 & 0.0555 & 0.1554 \\ \hline
FLinear & 0.0575 (-2.5\%) & 0.0346 (3.1\%) & 0.0573 (-3.2\%) & 0.1556 (-0.1\%) \\
FairFed & 0.0565 (-0.6\%) & 0.0405 (-13.3\%) & 0.0551 (0.7\%) & 0.1574 (-1.3\%) \\
FedFB & 0.0481 (14.2\%) & 0.0291 (18.6\%) & 0.0530 (4.5\%) & 0.1539 (1.0\%) \\
Reweight & 0.0491 (12.5\%) & 0.0294 (17.7\%) & 0.0527 (5.1\%) & 0.1546 (0.5\%) \\
Gaussian & 0.0791 (-40.9\%) & 0.0274 (23.3\%) & 0.0543 (2.2\%) & 0.1448 (6.8\%) \\
Uniform & 0.0594 (-5.8\%) & 0.0300 (16.1\%) & 0.0484 (12.8\%) & 0.1560 (-0.4\%) \\
\rowcolor{greyL}
\alg & 0.0470 (16.3\%) & 0.0262 (26.7\%) & 0.0381 (31.4\%) & 0.1358 (12.6\%) \\ \hline
\end{tabular}
}
 \subfloat[Income-Race.]
 {
 \begin{tabular}{|l|c|c|c|c|}
\hline
Method & EO & DP & CAL & CON \\ \hline
FedAvg & 0.3402 & 0.2076 & 0.1203 & 0.1262 \\ \hline
FLinear & 0.3643 (-7.1\%) & 0.2331 (-12.2\%) & 0.1182 (1.8\%) & 0.0982 (22.6\%) \\ 
FairFed & 0.3320 (2.4\%) & 0.2180 (-5.0\%) & 0.1150 (4.3\%) & 0.1252 (0.8\%) \\ 
FedFB & 0.3321 (2.4\%) & 0.1973 (5.0\%) & 0.1365 (-13.5\%) & 0.1458 (-15.8\%) \\ 
Reweight & 0.3660 (-7.6\%) & 0.2230 (-7.4\%) & 0.1350 (-12.5\%) & 0.1250 (1.0\%) \\ 
Gaussian & 0.3949 (-16.1\%) & 0.2676 (-29.0\%) & 0.1151 (4.3\%) & 0.0993 (21.6\%) \\ 
Uniform & 0.3202 (5.9\%) & 0.3093 (-48.0\%) & 0.1271 (-5.7\%) & 0.1228 (2.8\%) \\ 
\rowcolor{greyL}
\alg & 0.2980 (12.4\%) & 0.1863 (10.4\%) & 0.1147 (4.5\%) & 0.0789 (37.2\%) \\ \hline
\end{tabular}
 }
\vspace{0.08in}
\\
 \subfloat[MNIST Dataset.]
 {
\begin{tabular}{|l|c|c|c|c|}
\hline
Method & EO & DP & CAL & CON \\ \hline
FedAvg & 0.0205 & 0.1932 & 0.3761 & 0.0102 \\ \hline
FLinear & 0.0241 (-17.6\%) & 0.1871 (3.2\%) & 0.3102 (17.5\%) & 0.0085 (16.7\%) \\ 
FairFed & 0.0228 (-11.2\%) & 0.1923 (0.5\%) & 0.3724 (0.9\%) & 0.0082 (19.6\%) \\
FedFB & 0.0239 (-16.6\%) & 0.1849 (4.3\%) & 0.3721 (1.1\%) & 0.0090 (11.8\%) \\ 
Reweight & 0.0250 (-22.0\%) & 0.1742 (9.8\%) & 0.3792 (-0.8\%) & 0.0113 (10.8\%) \\
Gaussian & 0.0292 (-42.4\%) & 0.1891 (2.1\%) & 0.3831 (-1.9\%) & 0.0130 (-27.5\%) \\ 
Uniform & 0.0318 (-55.1\%) & 0.1887 (2.3\%) & 0.3925 (-4.4\%) & 0.0137 (-34.3\%) \\
\rowcolor{greyL}
\alg & 0.0137 (33.2\%) & 0.1814 (6.1\%) & 0.2726 (27.5\%) & 0.0078 (23.5\%) \\ \hline
\end{tabular}
\label{main_minst}
}
 \subfloat[CIFAR-10 Dataset.]
 {
\begin{tabular}{|l|c|c|c|c|}
\hline
Method & EO & DP & CAL & CON \\ \hline
FedAvg & 0.9886 & 0.3004 & 0.0622 & 0.4507 \\ \hline
FLinear & 0.8436 (14.7\%) & 0.2154 (28.3\%) & 0.0565 (9.2\%) & 0.3142 (30.3\%) \\
FairFed & 0.8075 (18.3\%) & 0.2348 (21.8\%) & 0.0623 (-0.2\%) & 0.3594 (20.3\%) \\ 
FedFB & 0.8875 (10.2\%) & 0.2653 (11.7\%) & 0.0784 (-26.1\%) & 0.3864 (14.3\%) \\ 

Reweight & 0.8121 (17.8\%) & 0.2095 (30.3\%) & 0.4501 (-623.0\%) & 0.2988 (33.7\%) \\ 
Gaussian & 0.9759 (1.3\%) & 0.2804 (6.7\%) & 0.1412 (-127.0\%) & 0.4301 (4.6\%) \\
Uniform & 1.0523 (-6.4\%) & 0.2960 (1.5\%) & 0.1538 (-147.3\%) & 0.4184 (7.2\%) \\
\rowcolor{greyL}
\alg & 0.7338 (25.8\%) & 0.1328 (55.8\%) & 0.0109 (82.5\%) & 0.0792 (82.4\%) \\ \hline
\end{tabular}
}
\label{main_results_fair}
\end{table*}

\subsection{Experimental Results}

\myparatight{Our proposed \alg is effective}%
Table~\ref{main_results_fair} reports the bias scores and fairness improvements of various debiasing methods across six datasets using multiple fairness metrics. For FedAvg, only the bias score is shown, while for other methods, the table presents both the bias score and the relative fairness improvement over FedAvg. A lower bias score indicates a fairer model, and a higher improvement reflects better debiasing performance.
Our method consistently outperforms the baselines. On the Income-Race dataset, it improves EO by 45.1\% and DP by 71.5\%. On CIFAR-10 with the ResNet-18 model, it achieves even larger gains, improving EO by 25.8\%, DP by 55.8\%, CAL by 82.5\%, and CON by 82.4\%.

Table~\ref{main_results_acc} in Appendix shows the test accuracy of the final global models using different debiasing methods across six datasets. FLinear, FairFed, FedFB, and Reweight are general-purpose methods that aim to improve fairness across all metrics, resulting in the same accuracy under each.
Our method effectively enhances fairness while maintaining high model utility. For instance, on the challenging CIFAR-10 dataset, \alg achieves test accuracy comparable to FedAvg, demonstrating that fairness can be improved without sacrificing performance.

\begin{figure*}[!t]
    \centering
    \includegraphics[width=\textwidth]{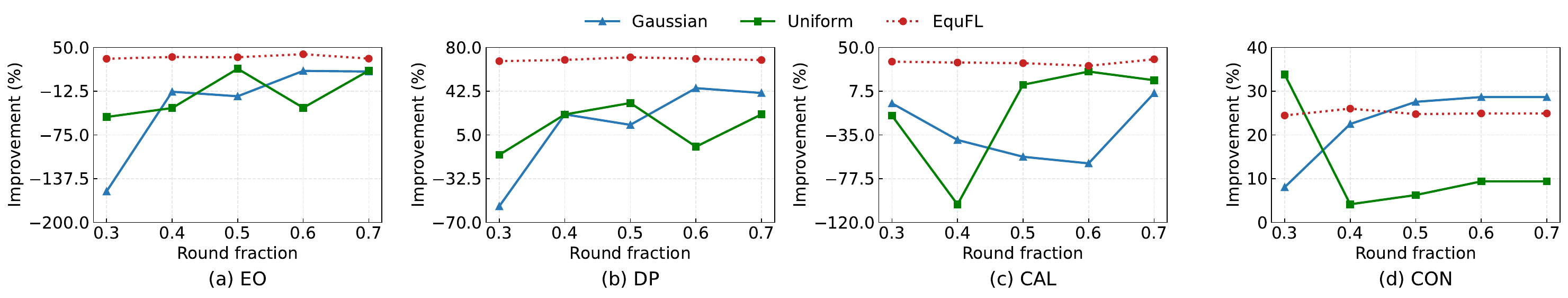}
    \caption{Impact of round fraction for calibrated update.}
    \label{fig:time}
\vspace{-.15in}
\end{figure*}

\myparatight{Impact of $\gamma$}%
In \alg, the server adds a calibrated update to the aggregated client updates, with $\gamma$ controlling the trade-off. Fig.~\ref{fig:gamma} in Appendix shows the impact of $\gamma$ on \alg and the Gaussian and Uniform baselines using the Income-Sex dataset.
As $\gamma$ increases, \alg shows steady and near-linear fairness improvements across all metrics, consistently outperforming the baselines. 
However, Gaussian and Uniform methods exhibit inconsistent fairness gains as $\gamma$ increases.

\myparatight{Impact of round fraction for calibrated update}%
This section studies how the fraction of training rounds used for generating calibrated updates affects performance. For example, if the server collects models for 40 rounds and generates calibrated updates in the remaining 60 of 100 total rounds, the round fraction is 60\%. Results for Gaussian, Uniform, and \alg are shown in Fig.~\ref{fig:time}.
\alg maintains stable fairness improvements across different round fractions, demonstrating robustness to this parameter. In contrast, Gaussian and Uniform baselines show significant performance fluctuations as the round fraction varies.

\myparatight{Impact of the size of Synthetic dataset}%
Fig.~\ref{fig:datasize} in Appendix shows how synthetic dataset size affects fairness. 
As the number of samples increases from 500 to 2000, fairness improves across metrics, with a sharp gain between 100 and 1000 samples. 
Beyond 1000, improvements plateau, indicating diminishing returns near 1500 to 2000.

\myparatight{Impact of total number of clients}%
Fig.~\ref{fig:clientnum} in Appendix shows the impact of client number on debiasing performance using the MNIST dataset, with clients ranging from 5 to 200. The Income-Sex dataset is excluded due to its fixed 51-client partition. Across all settings, \alg consistently outperforms baselines and remains stable as client numbers increase, indicating strong scalability and robustness.

\myparatight{Performance of \alg with complex aggregation rules}%
We use FedAvg as the default aggregation rule. To test \alg's compatibility with other strategies, we evaluate it under Median~\cite{yin2018byzantine}, Trimmed-mean~\cite{yin2018byzantine}, Multi-Krum~\cite{blanchard2017machine}, and DeepSight~\cite{rieger2022deepsight}. As shown in Table~\ref{tab:aggregation_results} (Appendix), \alg consistently improves fairness across all methods. 
For instance, Median reduces EO from 0.0613 to 0.0374 and DP from 0.0925 to 0.0493, showing strong versatility.

\myparatight{Impact of Non-IID}%
A key feature of FL is the Non-IID distribution of client data. Table~\ref{tab:non-iid} in Appendix examines this using MNIST, where each client holds data from only two or three labels. The Income-Sex dataset is excluded due to its inherent heterogeneity. Combined with Table~\ref{main_minst}, the results show that our method consistently reduces bias under different levels of data non-IIDness.

\myparatight{Transferability of different fairness metrics}%
Table~\ref{tab:transform_results} in Appendix shows the transferability of fairness metrics. For example, \alg-EO, which optimizes for equalized odds, also reduces DP bias from 0.0934 to 0.0627 (a 32.9\% improvement). However, improving one metric may sometimes increase bias in another, highlighting potential conflicts between fairness definitions~\cite{binns2020apparent,goethals2024beyond,mashiat2022trade}.

\myparatight{Server uses different networks to generate the synthetic dataset}%
As described in Section~\ref{section_Synthetic_data}, the server uses a neural network $f$, with an architecture different from the clients’, to generate synthetic data. Here, we examine how changing $f$'s architecture impacts performance. Details of the three architectures are in Appendix~\ref{sec:appendix_architectures}, and results are shown in Table~\ref{tab:model_arch} in Appendix. Across all architectures, our method consistently reduces bias, confirming its robustness.

%% file: discussion.tex

\section{Discussion and Limitations} 
\label{sec:discussion_limitation}

\myparatight{Server employs different strategies to collect global models}%
We examine two strategies for collecting global models over 30 training rounds. In the ``Discrete" strategy, the server randomly selects 30 rounds. In the ``Continuous" strategy, it collects models from the first 30 rounds. 
As shown in Table~\ref{tab:continuous_discrete} (Appendix), using early rounds improves bias reduction in synthetic dataset generation.

\myparatight{Compare \alg with other debiasing methods}%
Our experiments show that \alg effectively reduces bias and outperforms existing methods. To further assess its performance, we compare it with a regularization-based approach that adds a fairness term to each client's local objective. As shown in Table~\ref{tab:Regular} in Appendix, this method offers only minor fairness gains, significantly lower than those achieved by \alg.

\myparatight{Optimize multiple fairness metrics simultaneously}%
While \alg typically optimizes one fairness metric at a time, we extend it to handle multiple metrics simultaneously. In this setting, the server generates separate calibrated updates for EO, DP, CAL, and CON, and merges them using a multi-objective optimization technique. Table~\ref{tab:moo_results} (Appendix) shows that \alg-Multi, using MGDA~\cite{desideri2012multiple}, effectively reduces bias across all metrics.

\begin{figure}[!t]
    \centering
    \includegraphics[width=0.49\textwidth]{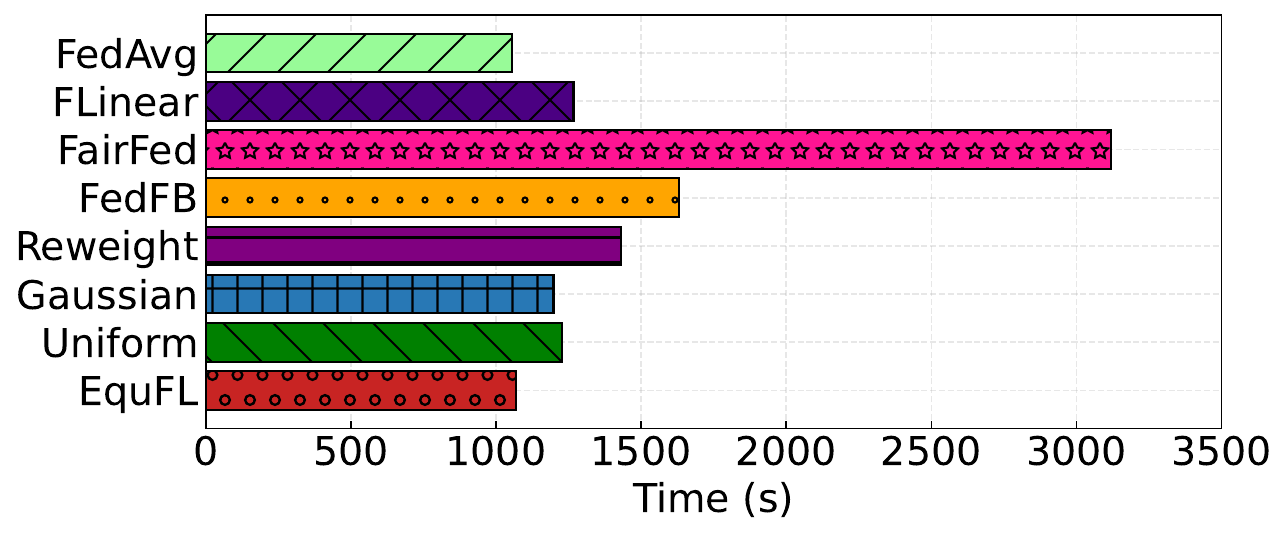}
    \caption{Computation costs.}
    \label{fig:cost_comm}
     \vspace{-4mm}
\end{figure}

\myparatight{Storage and computation cost for the server}%
In \alg, the server stores collected global models and the synthetic dataset, leading to modest storage overhead. As shown in Table~\ref{tab:storage_cost} in Appendix, for CIFAR-10, total storage is 450.02 MB, which is acceptable for modern servers.
Fig.~\ref{fig:cost_comm} shows the computation cost of different methods on the Income-Sex dataset using the EO metric.
While \alg involves additional steps for synthetic data and calibrated updates, its total computation time is similar to FedAvg.

\myparatight{Security concerns of \alg}%
This paper focuses on fairness in non-adversarial FL settings, where all clients behave honestly. Even when clients suffer from hardware failures and send unreliable updates, \alg remains effective. It is compatible with any Byzantine-robust aggregation rule, as its design is aggregation-agnostic.
As shown in Table~\ref{tab:aggregation_results}, \alg maintains strong performance and further improves fairness under robust schemes like Median and DeepSight, highlighting its versatility.

\myparatight{Privacy concerns of \alg}%
To generate the synthetic dataset, \alg requires the server to store selected global models, which may raise privacy concerns. However, these risks can be mitigated using techniques like  differential privacy~\cite{abadi2016deep}. 
As an example, we follow the standard DP-SGD~\cite{abadi2016deep} approach, where each client first clips its gradients to a fixed norm bound $C$, then adds Gaussian noise $\mathcal{N}(0, \sigma^2 C^2 I)$, with $I$ being the identity matrix. In our experiments, we set $C = 0.05$ and vary $\sigma$ in $\{0.1, 0.2, 0.3\}$ to explore different noise levels, and Income-Sex dataset is considered. 
Table~\ref{tab:dp_noise_levels} (Appendix) reports the performance of our method under these settings. For comparison, Table~\ref{tab:fedavg_dp_noise} (Appendix) shows the test accuracy of FedAvg (without calibrated updates). Results indicate that \alg remains effective at reducing system bias when moderate noise is applied. 
However, adding too much noise can harm model utility. For example, FedAvg’s test accuracy drops from 0.7491 without noise ($\sigma{=}0$) to 0.6960 when $\sigma{=}0.3$. This illustrates a key trade-off: differential privacy enhances data protection but may reduce model performance.
%

%% file: conclusion.tex

\section{Conclusion} \label{sec:conclusion}

In this paper, we proposed \alg, a server-side method that enhances fairness in FL by generating a calibrated update. Unlike prior approaches, \alg collects selected global models during training to build a synthetic dataset, which is then used to create a single calibrated update that reduces system bias.
We provided theoretical guarantees and validate \alg through extensive experiments, showing strong fairness improvements with minimal impact on accuracy.

\section*{Acknowledgments}
We thank the anonymous reviewers for their comments.

%% file: appendix.tex

\appendix

\begin{algorithm}[H]
\caption{DataSyn.}
\label{syn_data_gen}
\begin{algorithmic}[1]
\Require Global model checkpoints $\{\mathbf{w}^1, \mathbf{w}^2, \ldots, \mathbf{w}^s\}$, learning rate $\eta_t$, training iterations $\varkappa$, network $f$, parameter $\vartheta$.
\State Initialize $\mathbf{X}_{\text{syn}}^1$ and associate them with $\mathbf{Y}_{\text{syn}}^1$.
\For{$\varrho = 1$ to $\varkappa$}
    \State Randomly
    select two global models $\mathbf{w}^{\tau}$ and $\mathbf{w}^{\tau+\vartheta}$ from $\{\mathbf{w}^1, \mathbf{w}^2, \ldots, \mathbf{w}^s\}$.
   \State Train the network $f$ on the current synthetic dataset with $\mathbf{w}^{\tau}$ for $\vartheta$ steps to obtain $\overrightarrow{\mathbf{w}}$. 
    \State Compute $\Pi(\mathbf{X}_{\text{syn}}, \mathbf{Y}_{\text{syn}})$, followed by computing the gradients $\nabla_{\mathbf{X}_{\text{syn}}^{\varrho}} \Pi(\mathbf{X}_{\text{syn}}, \mathbf{Y}_{\text{syn}})$ and $\nabla_{\mathbf{Y}_{\text{syn}}^{\varrho}} \Pi(\mathbf{X}_{\text{syn}}, \mathbf{Y}_{\text{syn}})$.
  \State Update features and labels as $\mathbf{X}_{\text{syn}}^{\varrho+1}=\mathbf{X}_{\text{syn}}^{\varrho} - \eta_{\varrho} \cdot \nabla_{\mathbf{X}_{\text{syn}}^{\varrho}} \Pi(\mathbf{X}_{\text{syn}}, \mathbf{Y}_{\text{syn}})$, $\mathbf{Y}_{\text{syn}}^{\varrho+1}=\mathbf{Y}_{\text{syn}}^{\varrho} - \eta_{\varrho} \cdot \nabla_{\mathbf{Y}_{\text{syn}}^{\varrho}} \Pi(\mathbf{X}_{\text{syn}}, \mathbf{Y}_{\text{syn}})$. 
  \label{syn_data_update}
\EndFor \\
\Return $\mathcal{D}_{syn}$
\end{algorithmic}
\end{algorithm}

\begin{algorithm}[H]
	\caption{\alg.}
    \label{our_alg}
	\begin{algorithmic}[1]
		\renewcommand{\algorithmicrequire}{\textbf{Input:}}
		\renewcommand{\algorithmicensure}{\textbf{Output:}}
		\Require The $n$ clients with local training datasets $\mathcal{D}_i, i=1,2,\cdots,n$; aggregation rule $\mathsf{GAR}(\cdot)$; fairness metric $\mathcal{M}$; number of global training rounds $T$; learning rate $\eta_t$; network $f$; parameters $\gamma_t,s, \varkappa, \vartheta$.
		\Ensure Global model $\mathbf{w}^T$. 
		\State Random initialize $\mathbf{w}^1$.
        \State $\mathcal{S} \leftarrow \emptyset$.
		\For{$t=1,2,\cdots,T$}
		    \State // Step I (Global model synchronization).
		    \State The server sends the current global model $\mathbf{w}^t$ to all clients.
                 \If {$t \le s$}
                  \State   $\mathcal{S} \leftarrow \mathcal{S} \cup \{\mathbf{w}^t\}$.
                 \EndIf
	        \State // Step II (Local model training).
		    \For {each client $i=1,2,\cdots,n$ in parallel}
                \State Client $i$ updates its local model using $\mathbf{w}^t$ and its local data $\mathcal{D}_i$, then sends the update $\mathbf{g}^t_{i}$ to the server.
		    \EndFor
		    \State // Step III (Aggregation and global model updating).
                \If {$t \le s$}
                 \State  $\mathbf{w}^{t+1} = \mathbf{w}^t - \eta_t \cdot \mathsf{GAR}(\mathbf{g}^t_{1}, \mathbf{g}^t_{2}, \cdots, \mathbf{g}^t_{n}),$
                \EndIf
                \If {$t = s+1$}
                \label{Construct_one}
                \State  // Construct the synthetic dataset $\mathcal{D}_{\text{syn}}$. Note that $\mathcal{D}_{\text{syn}}$ is constructed only at round $s+1$.
                \State $\mathcal{D}_{\text{syn}}=\text{DataSyn}(\mathcal{S},\eta_t,\varkappa,f,\vartheta)$.
                \EndIf
                 \label{Construct_two}
                \If {$t \ge s+1$}
                \State Compute the calibrated update \( \mathbf{g}^t_{0} \) based on Eq.~(\ref{compute_g_0}).
                \State $\mathbf{w}^{t+1}= \mathbf{w}^t - \eta_t \cdot (\gamma_t \cdot \mathbf{g}^t_{0} + \mathsf{GAR}(\mathbf{g}^t_{1}, \mathbf{g}^t_{2}, \cdots, \mathbf{g}^t_{n}))$.
                \EndIf
		\EndFor\\
		\Return $\mathbf{w}$.
	\end{algorithmic} 
\end{algorithm}

\subsection{Optimization Problems for Other Fairness Metrics} \label{sec:appendix_1}

\subsubsection{Optimization Problem for Demographic Parity Metric}

The DP loss ensures that the model's predictions are independent of the sensitive attribute groups. It is defined as:
\begin{align}
     \mathcal{F}_{\text{DP}}(\mathbf{w}^{t+1}, \mathcal{D}_{\text{syn}}) = \sum_{h, k \in G } \bigg| \frac{1}{|\mathcal{D}_{\text{syn}}^{h}|}  
    \sum_{z \in \mathcal{D}_{\text{syn}}^{h}}  l(\mathbf{w}^{t+1}, z)   - 
    \frac{1}{|\mathcal{D}_{\text{syn}}^{k}|} \sum_{q \in \mathcal{D}_{\text{syn}}^{k}}  l(\mathbf{w}^{t+1}, q) \bigg|,
\end{align}
where $\mathbf{w}^{t+1}$ is still determined based on Eq.~(\ref{compute_w_t_1}), \(\mathcal{D}_{\text{syn}}^{h} = \{ z \in \mathcal{D}_{\text{syn}} : A = h \}\) denotes the subset of data points in the synthetic dataset \(\mathcal{D}_{\text{syn}}\) that are part of group \( h \).

\subsubsection{Optimization Problem for Calibration Metric}

The Calibration loss measures the difference in prediction errors between the overall positive class and each subgroup within it. It is defined as:
\begin{align}
      \mathcal{F}_{\text{CAL}}(\mathbf{w}^{t+1}, \mathcal{D}_{\text{syn}}) &= \sum_{h \in G}\bigg| \frac{1}{|\mathcal{D}_{\text{syn}}^1|}  
    \sum_{z \in \mathcal{D}_{\text{syn}}^1} l(\mathbf{w}^{t+1}, z)  - 
    \frac{1}{|\mathcal{D}_{\text{syn}}^{h,1}|} \sum_{q \in \mathcal{D}_{\text{syn}}^{h,1}}  l(\mathbf{w}^{t+1}, q) \bigg|,
\end{align}
where $\mathbf{w}^{t+1}$ is still determined based on Eq.~(\ref{compute_w_t_1}), \(\mathcal{D}_{\text{syn}}^{h,1} = \{ q \in \mathcal{D}_{\text{syn}} : A = h, Y = 1 \}\) denotes the subset of data points in the synthetic dataset \(\mathcal{D}_{\text{syn}}\) that are part of group \( h \) and have the true label \( y = 1 \), \(\mathcal{D}_{\text{syn}}^1 = \{ z \in \mathcal{D}_{\text{syn}} : Y = 1\}\) denotes the subset of data points in the synthetic dataset \(\mathcal{D}_{\text{syn}}\) that have the true label \( y = 1 \).

\subsubsection{Optimization Problem for Consistency Metric}

The CON loss assesses the consistency of model predictions for similar data points. For each sample \( z \in \mathcal{D}_{\text{syn}} \), identify its \( k \) nearest neighbors \( \mathcal{D}_k(z) \) based on feature similarity. The consistency loss is defined as:
\begin{align}
    \mathcal{F}_{\text{CON}}(\mathbf{w}^{t+1}, \mathcal{D}_{\text{syn}}) &= \frac{1}{|\mathcal{D}_{\text{syn}}|} \sum_{z \in \mathcal{D}_{\text{syn}}} \bigg| l(\mathbf{w}^{t+1}, z)  - \frac{1}{k} \sum_{q \in \mathcal{D}_k(z)}  l(\mathbf{w}^{t+1}, q) \bigg|,
\end{align}
where $\mathbf{w}^{t+1}$ is still determined based on Eq.~(\ref{compute_w_t_1}), \( \mathcal{D}_k(z) = \{ q \in \mathcal{D}_{\text{syn}} : q \text{ is among the } k \text{ nearest neighbors of } z \text{ in } \mathcal{D}_{\text{syn}} \}\) denotes the subset of data points in the synthetic dataset \(\mathcal{D}_{\text{syn}}\).

\subsection{Details of Assumption~\ref{assumption_1}}
\label{app_assumption_1}

The loss functions are \(\mu\)-strongly convex. For any \(\mathbf{w}_1, \mathbf{w}_2 \in \mathbb{R}^d\), the following inequalities hold:
\begin{gather*}
   \mathcal{L}(\mathbf{w}_1)  \geq \mathcal{L}(\mathbf{w}_2) + \nabla \mathcal{L}(\mathbf{w}_2)^{\top} (\mathbf{w}_1 - \mathbf{w}_2) + \frac{\mu}{2} \|\mathbf{w}_1 - \mathbf{w}_2\|^2 ,  \\
   \mathcal{L}_i(\mathbf{w}_1) 
     \geq \mathcal{L}_i(\mathbf{w}_2) + \nabla \mathcal{L}_i(\mathbf{w}_2)^{\top} (\mathbf{w}_1 - \mathbf{w}_2) + \frac{\mu}{2} \|\mathbf{w}_1 - \mathbf{w}_2\|^2,  \\
   \mathcal{F}(\mathbf{w}_1)  \geq \mathcal{F}(\mathbf{w}_2) + \nabla \mathcal{F}(\mathbf{w}_2)^{\top} (\mathbf{w}_1 - \mathbf{w}_2) + \frac{\mu}{2} \|\mathbf{w}_1 - \mathbf{w}_2\|^2,  \\
    \mathcal{F}_{\text{syn}}(\mathbf{w}_1)  \geq \mathcal{F}_{\text{syn}}(\mathbf{w}_2) + \nabla \mathcal{F}_{\text{syn}}(\mathbf{w}_2)^{\top} (\mathbf{w}_1 - \mathbf{w}_2)   
     + \frac{\mu}{2} \|\mathbf{w}_1 - \mathbf{w}_2\|^2.
\end{gather*}

The loss functions are \(\rho\)-smooth. For any \(\mathbf{w}_1, \mathbf{w}_2 \in \mathbb{R}^d\), the following inequalities are satisfied:
\begin{gather*}
   \mathcal{L}(\mathbf{w}_1) \leq \mathcal{L}(\mathbf{w}_2) + \nabla \mathcal{L}(\mathbf{w}_2)^{\top} (\mathbf{w}_1 - \mathbf{w}_2) + \frac{\rho}{2} \|\mathbf{w}_1 - \mathbf{w}_2 \|^2, \\
  \mathcal{L}_i(\mathbf{w}_1) \leq \mathcal{L}_i(\mathbf{w}_2) + \nabla \mathcal{L}_i(\mathbf{w}_2)^{\top} (\mathbf{w}_1 - \mathbf{w}_2) + \frac{\rho}{2} \|\mathbf{w}_1 - \mathbf{w}_2 \|^2, \\
    \mathcal{F}(\mathbf{w}_1)  \leq \mathcal{F}(\mathbf{w}_2) + \nabla \mathcal{F}(\mathbf{w}_2)^{\top} (\mathbf{w}_1 - \mathbf{w}_2) + \frac{\rho}{2} \|\mathbf{w}_1 - \mathbf{w}_2\|^2,  \\
    \mathcal{F}_{\text{syn}}(\mathbf{w}_1)  \leq \mathcal{F}_{\text{syn}}(\mathbf{w}_2) + \nabla \mathcal{F}_{\text{syn}}(\mathbf{w}_2)^{\top} (\mathbf{w}_1 - \mathbf{w}_2)  
     + \frac{\rho}{2} \|\mathbf{w}_1 - \mathbf{w}_2\|^2.
\end{gather*}

\subsection{Proof of Theorem~\ref{theorem1}}
\label{app_proof_theorem1}

According to Lemma~\ref{lemma2}, we have
    \begin{align}
        \Delta_{t+1} &\leq \left(1 - \mu(\eta_t + \eta_t \gamma_t )\right) \Delta_t + M_1 \eta_t^2 + M_2 \eta_t\gamma_t,
    \end{align}
where\(M_1 =4\rho\Gamma_1\) ,\(M_2 =2\Gamma_2\).
First, we use mathematical induction to prove the following inequality:
    \begin{align}
    \label{th1_delta}
    \Delta_t\leq\frac{\nu}{\varsigma+t}.
    \end{align}
$\textcircled{1}$ When $n = 1$
\begin{align}
\Delta_1 &\leq \frac{\nu}{\varsigma+1}\\
\nu&\geq \Delta_1(\varsigma+1)=\mathcal{Z}_1
\end{align}

$\textcircled{2}$ When $n = t+1$
    \begin{align}
\Delta_{t+1}&\leq(1-\frac{\mu \varpi}{t+\varsigma}-\frac{\mu \varpi}{(t+\varsigma)^{2}})\frac{\nu}{t+\varsigma}\: +\:M_1\frac{\varpi^{2}}{(t+\varsigma)^{2}}\:+\:M_2\frac{\varpi}{(t+\varsigma)^{2}}\\
&\stackrel{(a)}\leq ( 1- \frac {\mu \varpi }{t+ \varsigma } ) \frac \nu {t+ \varsigma } + M_1\frac {\varpi ^{2}}{( t+ \varsigma ) ^{2}} + M_2\frac \varpi {( t+ \varsigma ) ^{2}}\\
&=\frac{\nu\left(t+\varsigma-1\right)}{(t+\varsigma)^{2}}+\frac{((1-\mu \varpi)\nu+M_1 \varpi^{2}+M_2 \varpi)}{(t+\varsigma)^{2}}\\
&\stackrel{(b)}\leq \frac{\nu\left(t+\varsigma-1\right)}{(t+\varsigma)^{2}}\\
&\stackrel{(c)}\leq \frac {\nu} {t+ 1+ \varsigma },
\end{align}
where $(a)$ is due to $\frac{\mu \varpi}{(t+\varsigma)^{2}} \geq 0$, $(b)$ is due to
\begin{align}
\nu\geq \frac{M_1\varpi^{2}+M_2\varpi}{\mu \varpi-1} =\mathcal{Z}_2 \Rightarrow ((1-\mu \varpi)\nu+M_1 \varpi^{2}+M_2 \varpi)\leq0,
\end{align}
$(c)$ is due to
\begin{align}
(t+ \varsigma+1)(t+ \varsigma-1)<(t+ \varsigma)^2\Rightarrow \frac{\left(t+\varsigma-1\right)}{(t+\varsigma)^{2}}\leq\frac {1} {t+ 1+ \varsigma },
\end{align}
When $t=T$, we can get
\begin{align}
||\mathbf{w}^{T} - \mathbf{w}^*||^2=\Delta_T\leq\frac{\nu}{\varsigma+T}.
\end{align}
By the $\rho$-smooth of $\mathcal{L}$, we have
    \begin{align}
    \mathcal{L}(\mathbf{w}^{T})-\mathcal{L}^{*}
    &\stackrel{(a)}\leq \frac{\rho}{2}||\mathbf{w}^{T} - \mathbf{w}^*||^2
    \stackrel{(b)}\leq \frac{\rho}{2} \frac{\nu}{\varsigma+T}
    \end{align}
where $(a)$ is due to Assumption~\ref{assumption_1}, $(b)$ is based on Eq.~(\ref{th1_delta}).

\subsection{Proof of Theorem~\ref{theorem2}}
\label{app_proof_theorem2}

We have the following update rules:
\[
\left\{
\begin{aligned}
    \mathbf{w}^{t+1} &= \mathbf{w}^t - \eta_t  \sum_{i=1}^{n} \alpha_i \mathbf{g}^t_{i}  - \eta_t \gamma_t \mathbf{g}^t_{0}, \\
    \mathbf{v}^{t+1} &= \mathbf{w}^t - \eta_t  \sum_{i=1}^{n} \alpha_i \mathbf{g}^t_{i}.
\end{aligned}
\right.
\]
Therefore,
\begin{align}
    \mathbf{w}^{t+1} &= \mathbf{v}^{t+1} - \eta_t\gamma_t  \mathbf{g}^t_{0} = \mathbf{v}^{t+1} - \eta_t\gamma_t \nabla \mathcal{F}_{\text{syn}}(\mathbf{w}^t).
\end{align}

According to the Taylor expansion and the \( \rho \)-smoothness of \( \mathcal{F} \), we have:
\begin{align}
    \mathcal{F}(\mathbf{w}^{t+1}) &\leq \mathcal{F}(\mathbf{v}^{t+1}) + \langle \nabla \mathcal{F}(\mathbf{v}^{t+1}), \mathbf{w}^{t+1} - \mathbf{v}^{t+1} \rangle  + \dfrac{\rho}{2} \| \mathbf{w}^{t+1} - \mathbf{v}^{t+1} \|^2.
\end{align}

Substituting \( \mathbf{w}^{t+1} - \mathbf{v}^{t+1} = -\eta_t\gamma_t \nabla \mathcal{F}_{\text{syn}}(\mathbf{w}^t) \), we get:
\begin{align}
    \mathcal{F}(\mathbf{w}^{t+1}) - \mathcal{F}(\mathbf{v}^{t+1}) &\leq -\eta_t\gamma_t \langle \nabla \mathcal{F}(\mathbf{v}^{t+1}), \nabla \mathcal{F}_{\text{syn}}(\mathbf{w}^{t}) \rangle  + \dfrac{\rho \eta_t^2\gamma_t^2}{2} \| \nabla \mathcal{F}_{\text{syn}}(\mathbf{w}^t) \|^2.\label{eq:delta_G}
\end{align}
Define:
\begin{align}
    E &= -\eta_t\gamma_t \langle \nabla \mathcal{F}(\mathbf{v}^{t+1}), \nabla \mathcal{F}_{\text{syn}}(\mathbf{w}^{t}) \rangle + \dfrac{\rho \eta_t^2\gamma_t^2}{2} \| \nabla \mathcal{F}_{\text{syn}}(\mathbf{w}^t) \|^2.
\end{align}
We can further decompose \( \nabla \mathcal{F}(\mathbf{v}^{t+1}) \) as:
\begin{align}
    \nabla \mathcal{F}(\mathbf{v}^{t+1}) = \nabla \mathcal{F}(\mathbf{w}^t) + \delta,
\end{align}
where \( \delta = \nabla \mathcal{F}(\mathbf{v}^{t+1}) - \nabla \mathcal{F}(\mathbf{w}^t) \).

Since \( \mathcal{F} \) is \( \rho \)-smooth, we have:
\begin{align}
    \| \delta \| &= \| \nabla \mathcal{F}(\mathbf{v}^{t+1}) - \nabla \mathcal{F}(\mathbf{w}^t)\|\\ &\leq \rho \| \mathbf{v}^{t+1} - \mathbf{w}^t \| \\
    & = \rho \eta_t \| \sum_{i=1}^{n} \alpha_i \mathbf{g}^t_{i} \|\\
    & = \rho \eta_t \| \sum_{i=1}^{n} \alpha_i \nabla \mathcal{L}_{i}(\mathbf{w}^t) \|\\
    & = \rho \eta_t \|  \nabla \mathcal{L}(\mathbf{w}^t) \|\\
    &\label{them2_delta}\stackrel{(a)}\leq \rho \eta_t \sqrt{R},
\end{align}
where \( (a) \) is based on Assumption~\ref{assumption_3}.\\
\begin{align}
    \langle \nabla \mathcal{F}(\mathbf{v}^{t+1}), \nabla \mathcal{F}_{\text{syn}}(\mathbf{w}^{t}) \rangle &= \langle\nabla \mathcal{F}(\mathbf{w}^t) + \delta, \nabla \mathcal{F}_{\text{syn}}(\mathbf{w}^t) \rangle \\
    &   \label{eq:inner_product}= \underbrace{\langle\nabla \mathcal{F}(\mathbf{w}^t) , \nabla \mathcal{F}_{\text{syn}}(\mathbf{w}^t) \rangle}_{G} \underbrace{+ \langle \delta, \nabla \mathcal{F}_{\text{syn}}(\mathbf{w}^t) \rangle}_{H}.
\end{align}
\begin{align}
    G&= \langle\nabla \mathcal{F}(\mathbf{w}^t) , \nabla \mathcal{F}_{\text{syn}}(\mathbf{w}^t) \rangle\\
    &= \| \nabla \mathcal{F}_{\text{syn}}(\mathbf{w}^t)  \|^2 - \langle \nabla \mathcal{F}_{\text{syn}}(\mathbf{w}^t)  - \nabla \mathcal{F}(\mathbf{w}^t), \nabla \mathcal{F}_{\text{syn}}(\mathbf{w}^t)  \rangle \\
    &\stackrel{(a)}\geq \| \nabla \mathcal{F}_{\text{syn}}(\mathbf{w}^t)  \|^2 - \| \nabla \mathcal{F}_{\text{syn}}(\mathbf{w}^t)  - \nabla \mathcal{F}(\mathbf{w}^t) \| \cdot \| \nabla \mathcal{F}_{\text{syn}}(\mathbf{w}^t)  \| \\
    &\stackrel{(b)}\geq \| \nabla \mathcal{F}_{\text{syn}}(\mathbf{w}^t)  \|^2 - \epsilon \| \nabla \mathcal{F}_{\text{syn}}(\mathbf{w}^t)  \|, \label{them2_G}
\end{align}
where $(a)$ is based on the Cauchy-Schwarz inequality and $(b)$ is due to Assumption~\ref{assumption_4}.

We also have that:
\begin{align}
    H&=\langle \delta, \nabla \mathcal{F}_{\text{syn}}(\mathbf{w}^t) \rangle \stackrel{(c)}\geq -\| \delta \| \cdot \| \nabla \mathcal{F}_{\text{syn}}(\mathbf{w}^t) \| \stackrel{(d)} \geq -\rho \eta_t \sqrt{R} \| \nabla \mathcal{F}_{\text{syn}}(\mathbf{w}^t) \|, \label{them2_H}
\end{align}
where $(c)$ is based on the Cauchy-Schwarz inequality and $(d)$ is due to Eq.~(\ref{them2_delta}).

\begin{align}
     \langle \nabla \mathcal{F}(\mathbf{v}^{t+1}), \nabla \mathcal{F}_{\text{syn}}(\mathbf{w}^{t}) \rangle 
     &\stackrel{(e)}\geq \| \nabla \mathcal{F}_{\text{syn}}(\mathbf{w}^{t}) \|^2 - \epsilon \| \nabla \mathcal{F}_{\text{syn}}(\mathbf{w}^{t}) \|  -\rho \eta_t \sqrt{R} \| \nabla \mathcal{F}_{\text{syn}}(\mathbf{w}^{t}) \|\label{them2_GH},
\end{align}
where $(e)$ is based on the Eq.~(\ref{them2_H}) and Eq.~(\ref{them2_G}).\\
\begin{align}
    E &\stackrel{(a)}\leq -\eta_t\gamma_t( \| \nabla \mathcal{F}_{\text{syn}}(\mathbf{w}^{t}) \|^2 - \epsilon \| \nabla \mathcal{F}_{\text{syn}}(\mathbf{w}^{t}) \|  -\rho \eta_t \sqrt{R} \| \nabla \mathcal{F}_{\text{syn}}(\mathbf{w}^{t}) \|)+\dfrac{\rho \eta_t^2\gamma_t^2}{2} \| \nabla \mathcal{F}_{\text{syn}}(\mathbf{w}^t) \|^2\\
    &= \eta_t\gamma_t \left( \epsilon + \rho \eta_t \sqrt{R} \right) \| \mathcal{F}_{\text{syn}}(\mathbf{w}^t) \|   + (\dfrac{\rho \eta_t^2\gamma_t^2}{2}-\eta_t\gamma_t) \| \mathcal{F}_{\text{syn}}(\mathbf{w}^t) \|^2 .\label{eq:E_final}
\end{align}
We need to analyze the following expression and determine under what conditions \( E < 0 \): Let \( x = \| \mathcal{F}_{\text{syn}}(\mathbf{w}^t) \| \). 
\begin{align}
    I = (\dfrac{\rho \eta_t^2\gamma_t^2}{2}-\eta_t\gamma_t) x^2 + \eta_t\gamma_t \left( \epsilon + \rho \eta_t \sqrt{R} \right) x.
\end{align}
This is a quadratic function of \( x \) in the form:
\begin{align}
    I = q x^2 + p x,
\end{align}
where:
\begin{align}
    q = (\dfrac{\rho \eta_t^2\gamma_t^2}{2}-\eta_t\gamma_t),\ 
    p = \eta_t\gamma_t \left( \epsilon + \rho \eta_t \sqrt{R} \right).
\end{align}
Since \(\varsigma^2 > \frac{\rho \varpi}{2}\), we can have \( \eta_t\gamma_t > 0 \) and \(\eta_t\gamma_t < \dfrac{2}{\rho}\), \( q <0 \) .\\
Set \( I = 0 \):
\begin{align}
    q x^2 + p x = 0.
\end{align}
Solving for \( x \):
\begin{align}
    x_1 = 0 \quad \text{or} \quad x_2 = -\dfrac{p}{q}.
\end{align}
Compute:
\begin{align}
    x_2 &= -\dfrac{ \eta_t\gamma_t \left( \epsilon + \rho \eta_t \sqrt{R} \right) }{ \eta_t\gamma_t \left( \dfrac{\rho \eta_t\gamma_t}{2} - 1 \right) } = \dfrac{  \epsilon + \rho \eta_t \sqrt{R}  }{ 1- \dfrac{\rho \eta_t\gamma_t}{2} } .
\end{align}

To ensure \( I < 0 \) and \(x=\| \mathcal{F}_{\text{syn}}(\mathbf{w}^t) \|\geq \psi\), we need:
\begin{align}
     \psi > x_{2}.
\end{align}

Since the \(\eta_t = \frac{\varpi}{t+\varsigma}\) and \(\gamma_t =\frac{1}{t+\varsigma} \), the \(x_2\) is equal to
\begin{align}
x_2 &= \dfrac{ \epsilon + \rho \eta_t \sqrt{R} }{ 1 - \dfrac{\rho \eta_t \gamma_t}{2} } = \dfrac{ \epsilon + \dfrac{ \rho \varpi \sqrt{R} }{ t + \varsigma } }{ 1 - \dfrac{ \rho \varpi }{ 2 ( t + \varsigma )^2 } }.
\end{align}

To ensure \( \psi > x_2 \), we have:
\begin{align}
\psi &> \dfrac{ \epsilon + \dfrac{ \rho \varpi \sqrt{R} }{ t + \varsigma } }{ 1 - \dfrac{ \rho \varpi }{ 2 ( t + \varsigma )^2 } } \\
\psi \left( 1 - \dfrac{ \rho \varpi }{ 2 ( t + \varsigma )^2 } \right ) &> \epsilon + \dfrac{ \rho \varpi \sqrt{R} }{ t + \varsigma } \\
\psi - \dfrac{ \psi \rho \varpi }{ 2 ( t + \varsigma )^2 } &> \epsilon + \dfrac{ \rho \varpi \sqrt{R} }{ t + \varsigma } \\
\psi - \epsilon &> \dfrac{ \rho \varpi \sqrt{R} }{ t + \varsigma } + \dfrac{ \psi \rho \varpi }{ 2 ( t + \varsigma )^2 }.
\end{align}

At \( t = 0 \), the inequality becomes:
\begin{align}
\psi - \epsilon - \dfrac{ \rho \varpi \sqrt{R} }{ \varsigma } - \dfrac{ \psi \rho \varpi }{ 2 \varsigma^2 } &> 0.
\end{align}

Let \( C_1 = \rho \varpi \sqrt{R} \) and \( C_2 = \dfrac{ \psi \rho \varpi }{ 2 } \). Then:
\begin{align}
( \psi - \epsilon ) \varsigma^2 - C_1 \varsigma - C_2 &> 0.
\end{align}
Solving the quadratic equation, we can get the root which is greater than 0.
\begin{align}
root &= \dfrac{ C_1 + \sqrt{ C_1^2 + 4 ( \psi - \epsilon ) C_2 } }{ 2 ( \psi - \epsilon ) } 
= \dfrac{ \rho \varpi \sqrt{R} + \sqrt{ ( \rho \varpi )^2 R + 2 ( \psi - \epsilon ) \psi \rho \varpi } }{ 2 ( \psi - \epsilon ) }.
\end{align}
Therefore, the lower bound for \( b \) is:
\begin{align}
\varsigma &> \dfrac{ \rho \varpi \sqrt{R} + \sqrt{ ( \rho \varpi )^2 R + 2 ( \psi - \epsilon ) \psi \rho \varpi } }{ 2 ( \psi - \epsilon ) } \\
\varsigma &> \sqrt{ \dfrac{ \rho \varpi }{ 2 } }.
\end{align}
Combining both, we have:
\begin{align}
\varsigma &> \max \left\{ \sqrt{ \dfrac{ \rho \varpi }{ 2 } }, \ \dfrac{ \rho \varpi \sqrt{R} + \sqrt{ ( \rho \varpi )^2 R + 2 ( \psi - \epsilon ) \psi \rho \varpi } }{ 2 ( \psi - \epsilon ) } \right\}.
\end{align}

\subsection{Useful Technical Lemmas} 
\label{sec:appendix_lemma}

\begin{lemma}
    \label{lem:gradient_bound}
    Assume Assumption~\ref{assumption_1} holds. It follows that,
    \begin{align}
    \label{lemma1_l}
            \left\|\nabla \mathcal{L}_k \left(\mathbf{w}^t\right)\right\|^2 \leq 2\rho \left(\mathcal{L}_i\left(\mathbf{w}^t\right) - \mathcal{L}_i^*\right).
    \end{align}
    \begin{align}
    \label{lemma1_f}
            \left\|\nabla \mathcal{F}_{\text{syn}}\left(\mathbf{w}^t\right)\right\|^2 \leq 2\rho\left(\mathcal{F}_{\text{syn}}\left(\mathbf{w}^t\right) - \mathcal{F}_{\text{syn}}^*\right).
    \end{align}
    \label{lemma1}
\end{lemma}

\begin{proof}
We begin by utilizing the well-known inequality for \( \rho \)-smooth functions. For any \( \mathbf{x}, \mathbf{y} \in \mathbb{R}^n \), the following holds:
\begin{align}
    \mathcal{L}_i(\mathbf{y}) \leq \mathcal{L}_k(\mathbf{x}) + \nabla \mathcal{L}_i(\mathbf{x})^\top (\mathbf{y} - \mathbf{x}) + \frac{\rho}{2} \left\| \mathbf{y} - \mathbf{x} \right\|^2,
\end{align}
where \( \nabla \mathcal{L}_i(\mathbf{x})^\top \) denotes the transpose of the gradient of \( \mathcal{L}_i(\mathbf{x}) \) at \( \mathbf{x} \).

Next, we substitute \( \mathbf{y} = \mathbf{x} - \frac{1}{\rho} \nabla \mathcal{L}_i(\mathbf{x}) \) into the inequality:
\begin{align}
    \mathcal{L}_i(\mathbf{y}) \leq \mathcal{L}_i(\mathbf{x}) - \frac{1}{2\rho} \left\| \nabla \mathcal{L}_i(\mathbf{x}) \right\|^2.
\end{align}

Given that \( \mathcal{L}_i^* \) is the optimal value of the function \( \mathcal{L}_i \), we have \( \mathcal{L}_i^* \leq \mathcal{L}_i(\mathbf{y}) \). Therefore,
\begin{align}
    \mathcal{L}_i^* \leq \mathcal{L}_i(\mathbf{y}) \leq \mathcal{L}_i(\mathbf{x}) - \frac{1}{2\rho} \left\| \nabla \mathcal{L}_i(\mathbf{x}) \right\|^2.
\end{align}

Rearranging the terms yields:
\begin{align}
    \frac{1}{2\rho} \left\| \nabla \mathcal{L}_i(\mathbf{x}) \right\|^2 \leq \mathcal{L}_i(\mathbf{x}) - \mathcal{L}_i^*,
\end{align}\\
which can be equivalently expressed as:
\begin{align}
    \left\| \nabla \mathcal{L}_i\left(\mathbf{w}^t\right) \right\|^2 \leq 2\rho\left(\mathcal{L}_i\left(\mathbf{w}^t\right) - \mathcal{L}_i^*\right).
\end{align}

This completes the proof of inequality Eq.~(\ref{lemma1_l}). By following a similar procedure, inequality Eq.~(\ref{lemma1_f}) can be proven in the same manner.
\end{proof}

\begin{lemma}
    Consider the sequence \( \{\Delta_t\} \) defined as \( \Delta_t = \|\mathbf{w}^t - \mathbf{w}^*\|^2 \). Under Assumption~\ref{assumption_1} to Assumption~\ref{assumption_4}, if $\eta_t < \frac{1}{2\rho}$ and $\gamma_t \leq 1$,the following inequality holds:
    \begin{align}
        \Delta_{t+1} &\leq \left(1 - \mu(\eta_t + \eta_t \gamma_t )\right) \Delta_t + 4\rho\Gamma_1 \eta_t^2 + 2\Gamma_2 \eta_t\gamma_t.
    \end{align}
    \label{lemma2}
\end{lemma}

\begin{proof}
We begin by expanding the term \( \Delta_{t+1} \) as follows:
\begin{flalign}
\Delta_{t+1} &= \|\mathbf{w}^{t+1} - \mathbf{w}^*\|^2 \\
&= \|\mathbf{w}^t - \eta_t \sum_{i=1}^{n} \alpha_i \mathbf{g}^t_{i} - \eta_t \gamma_t \mathbf{g}^t_{0} - \mathbf{w}^*\|^2 \\
&\leq \|\mathbf{w}^t - \mathbf{w}^*\|^2 - 2\langle \mathbf{w}^t - \mathbf{w}^*, \eta_t \sum_{i=1}^{n} \alpha_i \mathbf{g}^t_{i}  + \eta_t \gamma_t \mathbf{g}^t_{0} \rangle  + \|\eta_t \sum_{i=1}^{n} \alpha_i \mathbf{g}^t_{i}  + \eta_t \gamma_t \mathbf{g}^t_{0}\|^2 \\
&\leq \underbrace{\|\mathbf{w}^t - \mathbf{w}^*\|^2}_{A} \underbrace{- 2\langle \mathbf{w}^t - \mathbf{w}^*, \eta_t \sum_{i=1}^{n} \alpha_i \mathbf{g}^t_{i}  + \eta_t \gamma_t \mathbf{g}^t_{0} \rangle}_{B}  \underbrace{+ 2\eta_t^2 \|\sum_{i=1}^{n} \alpha_i \mathbf{g}^t_{i} \|^2 + 2\eta_t^2\gamma_t^2 \|\mathbf{g}^t_{0}\|^2}_{C} .
\end{flalign}

Next, we decompose term \( B \) into two components \( B_1 \) and \( B_2 \) and analyze each separately:
\begin{align}
B &= B_1 + B_2,
\end{align}%
where%
\begin{align}
B_1 &=-2\eta_t\sum_{i=1}^{n}\alpha_{i}<\mathbf{w}^t-\mathbf{w}^{*},\mathbf{g}^t_{i}> &&\\
    &=-2\eta_t\sum_{i=1}^{n}\alpha_{i}<\mathbf{w}^t-\mathbf{w}^{*},\nabla \mathcal{L}_{i}(\mathbf{w}^t)> &&\\
    &\stackrel{(a)}{\leq}-2\eta_t \sum_{i=1}^{n}\alpha_{i}\left\{\mathcal{L}_{i}(\mathbf{w}^t)-\mathcal{L}_{i}(\mathbf{w}^{*})+\frac{\mu}{2}||\mathbf{w}^t-\mathbf{w}^{*}||^{2}\right\} &&\\
    &=\sum_{i=1}^{n}\alpha_{i}\left(-\mu\eta_t ||\mathbf{w}^t-\mathbf{w}^{*}||^{2}-2\eta_t \left(\mathcal{L}_{i}(\mathbf{w}^t)-\mathcal{L}_{i}(\mathbf{w}^{*})\right)\right) &&\\
    &\label{lamma1_b1}=-\mu\eta_t ||\mathbf{w}^t-\mathbf{w}^{*}||^{2}-2\eta_t \sum_{i=1}^{n}\alpha_{i}\left(\mathcal{L}_{i}(\mathbf{w}^t)-\mathcal{L}_{i}(\mathbf{w}^{*})\right), &&
\end{align}
where inequality \( (a) \) follows from the strong convexity of \( \mathcal{L}_k \).

Similarly, for \( B_2 \), we have:
\begin{align}
B_2 &= -2\eta_t \gamma_t \langle \mathbf{w}^t - \mathbf{w}^*, \nabla \mathcal{F}_{\text{syn}}(\mathbf{w}^t) \rangle \\
&\label{lamma1_b2}\stackrel{(b)}{\leq} -\mu \eta_t \gamma_t \|\mathbf{w}^t - \mathbf{w}^*\|^2 - 2\eta_t \gamma_t \left( \mathcal{F}_{\text{syn}}(\mathbf{w}^t) - \mathcal{F}_{\text{syn}}(\mathbf{w}^*) \right),
\end{align}
where inequality \( (b) \) follows from the strong convexity of \( \mathcal{F}_{\text{syn}}\).

Next, we consider term \( C \) and decompose it into two components \( C_1 \) and \( C_2 \) as follows:
\begin{flalign}
C &= C_1 + C_2,
\end{flalign}
where
\begin{align}
C_1 &= 2\eta_t^2 \|\sum_{i=1}^{n} \alpha_i \mathbf{g}^t_{i} \|^2  \\
&= 2\eta_t^2 \|\sum_{i=1}^{n} \alpha_i \nabla  \mathcal{L}_{i}(\mathbf{w}^t) \|^2  \\
&\stackrel{(c)}{\leq} 2\eta_t^2\sum_{i=1}^{n} \alpha_i  \| \nabla  \mathcal{L}_{i}(\mathbf{w}^t)\|^2 \\
&\label{lamma1_c1} \stackrel{(d)}{\leq} 4\rho\eta_t^2 \sum_{i=1}^{n} \alpha_i \left( \mathcal{L}_{i}(\mathbf{w}^t) - \mathcal{L}_{i}^* \right),
\end{align}
and
\begin{align}
C_2 &= 2\eta_t^2\gamma_t^2 \|\mathbf{g}^t_{0}\|^2 \\
&= 2\eta_t^2\gamma_t^2 \|\nabla \mathcal{F}_{\text{syn}}(\mathbf{w}^t)\|^2 \\
&\label{lamma1_c2} \stackrel{(e)}{\leq} 4\rho\eta_t^2\gamma_t^2 \left( \mathcal{F}_{\text{syn}}(\mathbf{w}^t) - \mathcal{F}_{\text{syn}}^* \right),
\end{align}
where inequalities \( (c) \) is based on the convexity and \( (d) \) \( (e) \) follow from Lemma~\ref{lemma1}.

We now combine \( B_1 \) and \( C_1 \) into a single term denoted \( \text{Part}_1 \), and \( B_2 \) and \( C_2 \) into a term denoted \( \text{Part}_2 \). This leads to:
\begin{align}
\Delta_{t+1} &\leq A + B + C \\
&= A + (B_1 + C_1) + (B_2 + C_2) \\
&= A + \text{Part}_1 + \text{Part}_2.
\end{align}
\begin{align}
Part_{1}&=B_{1}\:+\:C_{1}\:\\
& \stackrel{(a)}\leq -\mu\eta_t ||\mathbf{w}^t-\mathbf{w}^{*}||^{2}-2\eta_t \sum_{i=1}^{n}\alpha_{i}\left(\mathcal{L}_{i}(\mathbf{w}^t)-\mathcal{L}_{i}(\mathbf{w}^{*})\right)+4\rho\eta_t^{2}\sum_{i=1}^{n}\alpha_{i}\left(\mathcal{L}_{i}(\mathbf{w}^t)-\mathcal{L}_{i}^{*}\right)\\
&=-\mu \eta_t ||\mathbf{w}^t-\mathbf{w}^{*}||^{2}+\left(4\rho\eta_t^{2}-2\eta_t\right)\sum_{i=1}^{n}\alpha_{i}\left(\mathcal{L}_{i}(\mathbf{w}^t)-\mathcal{L}_i^{*}\right)\notag+2\eta_t\sum_{i=1}^{n}\alpha_{i}\left(\mathcal{L}_{i}\left(\mathbf{w}^*\right)-\mathcal{L}_i^{*}\right)\\
&=-\mu \eta_t ||\mathbf{w}^t-\mathbf{w}^{*}||^{2}+ D \label{lemma3_p1}
\end{align}
where $(a)$ is based on Eq.~(\ref{lamma1_b1}) and Eq.~(\ref{lamma1_c1}).
\begin{align}
D&=\left(4\rho\eta_t^{2}-2\eta_t\right)\sum_{i=1}^{n}\alpha_{i}\left(\mathcal{L}_{i}(\mathbf{w}^t)-\mathcal{L}_{i}^{*}\right)+2\eta_t\sum_{i=1}^{n}\alpha_{i}\left(\mathcal{L}_{i}\left(\mathbf{w}^*\right)-\mathcal{L}_i^{*}\right)\\
&=\left(4\rho\eta_t^{2}-2\eta_t\right)\left(\mathcal{L}(\mathbf{w}^t)-\sum_{i=1}^{n}\alpha_{i}\mathcal{L}_{i}^{*}\right) +2\eta_t\left(\mathcal{L}^*-\sum_{i=1}^{n}\alpha_{i}\mathcal{L}_{i}^{*}\right)\\
&=\left(4\rho\eta_t^{2}-2\eta_t\right)\left(\mathcal{L}(\mathbf{w}^t)-\mathcal{L}^{*}\right) +4\rho\eta_t^{2}\left(\mathcal{L}^*-\sum_{i=1}^{n}\alpha_{i}\mathcal{L}_{i}^{*}\right)\\
&=\left(4\rho\eta_t^{2}-2\eta_t\right)\left(\mathcal{L}(\mathbf{w}^t)-\mathcal{L}^{*}\right)+4\rho\eta_t^{2}\Gamma_1\\
&\label{lemma_d}\stackrel{(b)}\leq 4\rho\eta_t^{2}\Gamma_1,
\end{align}
where $(b)$ is due to the following facts:
\begin{itemize}
    \item $\eta_t < \frac{1}{2\rho}\Rightarrow 4\rho\eta_t^{2}-2\eta_t <0$
    \item $\mathcal{L}^* = \min (\mathcal{L})\Rightarrow \mathcal{L}\left(\mathbf{w}^t\right)-\mathcal{L}^{*}>0$.
\end{itemize}

So we can get:
\begin{align}
Part_1& =  -\mu\eta_t ||\mathbf{w}^t-\mathbf{w}^{*}||^{2} + D\ \label{lemma1_p1}\stackrel{(c)}\leq -\mu\eta_t ||\mathbf{w}^t-\mathbf{w}^{*}||^{2} + 4\rho\eta_t^{2}\Gamma_1,
\end{align}
where $(c)$ is due to Eq.~(\ref{lemma_d}).

Similarly, one has that:
\begin{align}
Part_2&= B_{2}+C_{2}\\
&\stackrel{(a)}\leq - 2\eta_t \gamma_t \left( \mathcal{F}_{\text{syn}}(\mathbf{w}^t) - \mathcal{F}_{\text{syn}}(\mathbf{w}^*) \right)-\mu \eta_t \gamma_t \|\mathbf{w}^t - \mathbf{w}^*\|^2+4\rho\eta_t^2\gamma_t^2 \left( \mathcal{F}_{\text{syn}}(\mathbf{w}^t) - \mathcal{F}_{\text{syn}}^* \right)\\
&=\left(4\rho\eta_t^{2}\gamma_t^{2}-2\eta_t\gamma_t\right)\left( \mathcal{F}_{\text{syn}}(\mathbf{w}^t) - \mathcal{F}_{\text{syn}}^* \right) -\mu \eta_t \gamma_t \|\mathbf{w}^t - \mathbf{w}^*\|^2+2\eta_t\gamma_t \left(\mathcal{F}_{\text{syn}}(\mathbf{w}^*) - \mathcal{F}_{\text{syn}}^*\right)\\
&\label{lemma1_p2}\stackrel{(b)}\leq -\mu \eta_t \gamma_t \|\mathbf{w}^t - \mathbf{w}^*\|^2+2\eta_t\gamma_t\Gamma_2,
\end{align}
where $(a)$ is based on Eq.~(\ref{lamma1_b2}) Eq.~(\ref{lamma1_c2}), $(b)$ is due to the following facts:
\begin{itemize}
    \item $\eta_t\gamma_t < \frac{1}{2\rho}\Rightarrow 4\rho\eta_t^{2}\gamma_t^{2}-2\eta_t\gamma_t <0$
    \item $\mathcal{F}_{\text{syn}}^* = \min(\mathcal{F}_{\text{syn}})\Rightarrow \mathcal{F}_{\text{syn}}(\mathbf{w}^t) - \mathcal{F}_{\text{syn}}^*>0$
    \item $ \mathcal{F}_{\text{syn}}(\mathbf{w}^*) - \mathcal{F}_{\text{syn}}^* =\Gamma_2$.
\end{itemize}

By integrating the above results, one has that:
\begin{align}
\Delta_{t+1}&=A+Part_1+Part_2\\
&\stackrel{(a)}\leq(1-\mu \eta_t-\mu\eta_t\gamma_t)\|\mathbf{w}^t-\mathbf{w}^*\|^{2}+\left(4\rho\eta_t^{2}\Gamma_1+2\eta_t\gamma_t \Gamma_{2}\right)\\
&=(1-\mu \eta_t-\mu\eta_t\gamma_t)\Delta_t+\left(4\rho\eta_t^{2}\Gamma_1+2\eta_t\gamma_t \Gamma_{2}\right).
\end{align}
where $(a)$ is based on Eq.~(\ref{lemma1_p1}) Eq.~(\ref{lemma1_p2}).
\end{proof}

\subsection{Details of Comparison Debiasing Methods} 
\label{sec:appendix_baseline}

\myparatight{Fair linear representation (FLinear)~\cite{he2020geometric}} Each client applies a pre-processing debiasing strategy known as fair linear representations, designed to mitigate bias in the dataset before model training.

\myparatight{FairFed~\cite{ezzeldin2023fairfed}}In FairFed, each client debiases its local dataset and evaluates global model fairness, collaborating with the server to adjust aggregation weights and enhance overall fairness.

\myparatight{FedFB~\cite{zeng2021improving}}A method that adapts FedAvg to achieve centralized fair learning by incorporating fairness constraints.

\myparatight{Local reweighting (Reweight)~\cite{kamiran2012data}}A preprocessing technique that reweights training samples locally to mitigate discrimination.

\myparatight{Gaussian}The server creates a calibrated update by sampling from a normal distribution with a mean of 0 and a standard deviation of 2. This generated update is then added to the aggregated update from the clients.

\myparatight{Uniform}The server produces a calibrated update by drawing random values for each dimension from a uniform distribution within the interval \([-2, 2]\). This randomly generated update is then combined with the aggregated client updates to form the final update.

\subsection{Details of Neural Network Architectures} 
\label{sec:appendix_architectures}

\myparatight{StandardMLP}A conventional Multi-Layer Perceptron with one hidden layer of 64 units, serving as our default architecture.

\myparatight{DeepMLP}A deeper version of the model features two hidden layers, the first with 64 units and the second with 32 units, increasing the model's depth while keeping the overall parameter count comparable to that of StandardMLP.

\myparatight{WideMLP}A wider architecture with one hidden layer of 128 units, doubling the width of StandardMLP while keeping the same depth.

\subsection{Details of Parameter Settings} 
\label{sec:appendix_settings}

For model training, we employ a two-layer neural network on the Income-Sex, Employment-Sex, Health-Sex, and Income-Race datasets, a two-layer CNN for MNIST, and a complex ResNet-18~\cite{he2016deep} model for CIFAR-10. Learning rate and batch size are set to 0.1 and 64 for the first four datasets, 0.01 and 32 for MNIST, and 0.002 and 16 for CIFAR-10. Training involves 100 communication rounds for the first four datasets, 30 rounds for MNIST, and 20 rounds for CIFAR-10.

\begin{table*}[t]
\centering
\footnotesize
\addtolength{\tabcolsep}{-1.565pt}
\caption{Test accuracy of the final global model learned using various debiasing methods.}
 \subfloat[Income-Sex.]
 {
\begin{tabular}{|l|c|c|c|c|}
\hline
Method & EO & DP & CAL & CON \\ \hline
FedAvg & 0.7491 & 0.7491 & 0.7491 & 0.7491 \\ \hline
FLinear  & 0.7400  & 0.7400  & 0.7400  & 0.7400  \\
FairFed  & 0.7532 & 0.7532 &0.7532  & 0.7532 \\
FedFB  & 0.7422 & 0.7422 & 0.7422 & 0.7422 \\
Reweight   & 0.7386 & 0.7386 & 0.7386 & 0.7386 \\
Gaussian  & 0.7005 & 0.7005 & 0.7005 & 0.7005 \\
Uniform  & 0.7223 & 0.7223 & 0.7223 & 0.7223 \\
\rowcolor{greyL}
\alg & 0.7076 & 0.7043 & 0.7122 & 0.7259 \\ \hline
\end{tabular}
}
 \subfloat[Employment-Sex.]
 {
\begin{tabular}{|l|c|c|c|c|}
\hline
Method & EO & DP & CAL & CON \\ \hline
FedAvg & 0.7095 & 0.7095 & 0.7095 & 0.7095 \\ \hline
FLinear  & 0.7030  & 0.7030  &0.7030  & 0.7030  \\
FairFed  & 0.6043 & 0.6043 & 0.6043 & 0.6043 \\
FedFB  & 0.7000 & 0.7000 & 0.7000 & 0.7000 \\
Reweight   & 0.7012 & 0.7012 & 0.7012 & 0.7012 \\
Gaussian  & 0.7034 & 0.7034 & 0.7034 & 0.7034 \\
Uniform  & 0.6846 & 0.6846 & 0.6846 & 0.6846 \\
\rowcolor{greyL}
\alg & 0.7049 & 0.7057 & 0.7060 &0.7061  \\ \hline
\end{tabular}
}
 \subfloat[Health-Sex.]
 {
\begin{tabular}{|l|c|c|c|c|}
\hline
Method & EO & DP & CAL & CON \\ \hline
FedAvg & 0.8243 & 0.8243 & 0.8243 & 0.8243 \\ \hline
FLinear  & 0.8242  & 0.8242 &0.8242  & 0.8242  \\
FairFed  & 0.8202 & 0.8202 & 0.8202 & 0.8202 \\
FedFB  & 0.8020 & 0.8020 & 0.8020 & 0.8020 \\
Reweight   & 0.8106 & 0.8106 & 0.8106 & 0.8106 \\
Gaussian  & 0.8106 & 0.8106 & 0.8106 & 0.8106 \\
Uniform  & 0.8046 & 0.8046 & 0.8046 & 0.8046 \\
\rowcolor{greyL}
\alg & 0.8170 & 0.8170 & 0.8170 & 0.8170 \\ \hline
\end{tabular}
}
\vspace{0.08in}
\\
 \subfloat[Income-Race.]
 {
\begin{tabular}{|l|c|c|c|c|}
\hline
Method & EO & DP & CAL & CON \\ \hline
FedAvg & 0.7490 & 0.7490 & 0.7490 & 0.7490 \\ \hline
FLinear  & 0.7480  & 0.7480  &0.7480  & 0.7480  \\
FairFed  & 0.7480 & 0.7480 & 0.7480 & 0.7480 \\
FedFB  & 0.7410 & 0.7410 & 0.7410 & 0.7410 \\
Reweight   & 0.7320 & 0.7320 & 0.7320 & 0.7320 \\
Gaussian  & 0.7008 & 0.7008 & 0.7008 & 0.7008 \\
Uniform  & 0.7193 & 0.7193 & 0.7193 & 0.7193 \\
\rowcolor{greyL}
\alg & 0.7230 & 0.7329 & 0.7317 & 0.7010 \\ \hline
\end{tabular}
}
 \subfloat[MNIST Dataset.]
{
\begin{tabular}{|l|c|c|c|c|}
\hline
Method & EO & DP & CAL & CON \\ \hline
FedAvg & 0.9685 & 0.9685 & 0.9685 & 0.9685 \\ \hline
FLinear  & 0.9523 & 0.9523 & 0.9523 & 0.9523 \\  
FairFed  & 0.9562 & 0.9562 & 0.9562 & 0.9562 \\ 
FedFB  & 0.9550 & 0.9550 & 0.9550 & 0.9550 \\ 
Reweight & 0.9631 & 0.9631 & 0.9631 & 0.9631 \\ 
Gaussian & 0.9114 & 0.9114 & 0.9114 & 0.9114 \\ 
Uniform & 0.9102 & 0.9102 & 0.9102 & 0.9102 \\
\rowcolor{greyL}
\alg & 0.9670 & 0.9672 & 0.9673 & 0.9684 \\ \hline
\end{tabular}
}
 \subfloat[CIFAR-10 Dataset.]
{
\begin{tabular}{|l|c|c|c|c|}
\hline
Method & EO & DP & CAL & CON \\ \hline
FedAvg & 0.7915 & 0.7915 & 0.7915 & 0.7915 \\ \hline
FLinear  & 0.7632 & 0.7632 & 0.7632 & 0.7632 \\
FairFed  & 0.7709 & 0.7709 & 0.7709 & 0.7709 \\
FedFB  & 0.7525 & 0.7525 & 0.7525 & 0.7525 \\
Reweight & 0.7680 & 0.7680 & 0.7680 & 0.7680 \\
Gaussian & 0.6132 & 0.6132 & 0.6132 & 0.6132 \\
Uniform & 0.5828 & 0.5828 & 0.5828 & 0.5828 \\
\rowcolor{greyL}
\alg & 0.7532 & 0.7643 & 0.7778 & 0.7596 \\ \hline
\end{tabular}
}
\label{main_results_acc}
\end{table*}

\begin{table*}[t]
\centering
\footnotesize
\addtolength{\tabcolsep}{-3.065pt}
\caption{Results of \alg across various fairness metrics, where the server employs complex aggregation rules to combine client updates.}
\label{tab:aggregation_results}
 \subfloat[Median.]
 {
\begin{tabular}{|l|c|c|c|c|}
\hline
Method & EO & DP & CAL & CON \\ \hline
Median & 0.0613 & 0.0925 & 0.0355 & 0.1252 \\ \hline
\rowcolor{greyL}
\alg & 0.0374 (39.0\%) & 0.0493 (46.7\%) & 0.0315(11.3\%) & 0.1230 (1.8\%) \\ \hline
\end{tabular}
}
 \subfloat[Trimmed-mean.]
 {
\begin{tabular}{|l|c|c|c|c|}
\hline
Method & EO & DP & CAL & CON \\ \hline
Trim & 0.0613 & 0.0918 & 0.0332 & 0.1264 \\ \hline
\rowcolor{greyL}
\alg & 0.0363 (40.8\%) & 0.0475 (48.3\%) &0.0294(11.4\%) & 0.1211 (4.2\%) \\ \hline
\end{tabular}
}
\vspace{0.08in}
\\
 \subfloat[Multi-Krum.]
 {
 \addtolength{\tabcolsep}{-1.65pt}
\begin{tabular}{|l|c|c|c|c|}
\hline
Method & EO & DP & CAL & CON \\ \hline
Multi-krum & 0.0570 & 0.0922 & 0.0303 & 0.1285 \\ \hline
\rowcolor{greyL}
\alg & 0.0355 (37.7\%) & 0.0453 (50.9\%) &  0.0248(18.2\%) & 0.1132 (11.9\%) \\ \hline
\end{tabular}
}
 \subfloat[DeepSight.]
 {
 \addtolength{\tabcolsep}{-0.9pt}
\begin{tabular}{|l|c|c|c|c|}
\hline
Method & EO & DP & CAL & CON \\ \hline
DeepSight & 0.0192 & 0.0093 & 0.1215 & 0.0563 \\ \hline
\rowcolor{greyL}
\alg &  0.0112(41.6\%) & 0.0087 (6.5\%) &  0.0966(20.5\%) & 0.0132 (76.6\%) \\ \hline
\end{tabular}
}
\label{other_rule}
\end{table*}

\begin{table*}[h]
\centering
 \scriptsize
\addtolength{\tabcolsep}{-2.045pt}
\caption{Impact of degree of Non-IID.}
\label{tab:non-iid}
 \subfloat[Each client only has two labeled training examples.]
{
\begin{tabular}{|l|c|c|c|c|}
\hline
Method & EO & DP & CAL & CON \\ \hline
FedAvg & 0.0268 & 0.1961 & 0.3817 & 0.0091 \\ \hline
FLinear & 0.0252 (6.0\%) & 0.1902 (3.0\%) & 0.3236 (15.2\%) & 0.0089 (2.2\%) \\
FairFed & 0.0248 (7.5\%) & 0.1963 (-0.1\%) & 0.3781 (0.9\%) & 0.0085 (6.6\%) \\
FedFB & 0.0257 (4.1\%) & 0.1894 (3.4\%) & 0.0380 (90.0\%) & 0.0089 (2.2\%) \\
Reweight & 0.0278 (-3.7\%) & 0.0213 (89.1\%) & 0.3816 (0.0\%) & 0.0117 (-28.6\%) \\ 
Gaussian & 0.0319 (-19.0\%) & 0.1929 (1.6\%) & 0.3847 (-0.8\%) & 0.0130 (-42.9\%) \\ 
Uniform & 0.0336 (-25.4\%) & 0.1929 (1.6\%) & 0.3977 (-4.2\%) & 0.0138 (-51.6\%) \\ 
\rowcolor{greyL}
\alg & 0.0152 (43.3\%) & 0.1841 (6.1\%) & 0.2789 (26.9\%) & 0.0072 (20.9\%) \\ \hline
\end{tabular}
}
 \subfloat[Each client only has three labeled training examples.]
{
\begin{tabular}{|l|c|c|c|c|}
\hline
Method & EO & DP & CAL & CON \\ \hline
FedAvg & 0.0253 & 0.1975 & 0.3842 & 0.0089 \\ \hline
FLinear & 0.0262 (-3.6\%) & 0.1892 (4.2\%) & 0.3253 (15.3\%) & 0.0088 (1.1\%) \\
FairFed & 0.0255 (-0.8\%) & 0.1956 (1.0\%) & 0.3760 (2.1\%) & 0.0084 (5.6\%) \\
FedFB & 0.0268 (-5.9\%) & 0.1897 (3.9\%) & 0.0385 (90.0\%) & 0.0087 (2.2\%) \\
Reweight & 0.0284 (-12.3\%) & 0.0226 (88.6\%) & 0.3805 (1.0\%) & 0.0112 (-25.8\%) \\ 
Gaussian & 0.0342 (-35.2\%) & 0.1931 (2.2\%) & 0.3858 (-0.4\%) & 0.0132 (-48.3\%) \\ 
Uniform & 0.0350 (-38.3\%) & 0.1934 (2.1\%) & 0.3940 (-2.5\%) & 0.0139 (-56.2\%) \\ 
\rowcolor{greyL}
\alg & 0.0157 (38.0\%) & 0.1856 (6.0\%) & 0.2811 (26.8\%) & 0.0076 (14.6\%) \\ \hline
\end{tabular}
}
\end{table*}

\begin{table}[h!]
\centering
\footnotesize
\caption{Transferability of different fairness metrics.}
\label{tab:transform_results}
\begin{tabular}{|l|c|c|c|c|}
\hline
Method & EO & DP & CAL & CON \\
    \hline
    FedAvg & 0.0611 & 0.0934 & 0.0343 & 0.1281 \\
    \hline
    \alg-EO & 0.0335 (45.1\%) & 0.0627 (32.9\%) & 0.0420 (-22.4\%) & 0.0527 (58.9\%) \\ 
    \alg-DP & 0.0421 (31.1\%) & 0.0266 (71.5\%) & 0.0393 (-14.6\%) & 0.0676 (47.2\%)  \\ 
    \alg-CAL & 0.0481 (21.3\%) & 0.7122 (-662.5\%) & 0.0224 (34.6\%) & 0.0537 (58.1\%)  \\ 
    \alg-CON & 0.1006 (-64.6\%) & 0.1413 (-51.3\%) & 0.1229 (-258.3\%) & 0.0948 (25.9\%)  \\
    \hline
\end{tabular}
\end{table}

\begin{table}[h!]
\centering
\footnotesize
\caption{Server uses different networks to generate the synthetic
dataset.}
\label{tab:model_arch}
\begin{tabular}{|l|c|c|c|c|}
\hline
Method & EO & DP & CAL & CON \\ \hline
FedAvg & 0.0611 & 0.0934 & 0.0343 & 0.1281 \\ \hline
WideMLP & 0.0339 (44.5\%) & 0.0263 (71.8\%) & 0.0213 (49.6\%) & 0.0951 (25.7\%) \\
DeepMLP & 0.0337 (44.8\%) & 0.0262 (71.9\%) & 0.0229 (50.7\%) & 0.0953 (25.6\%) \\
StandardMLP & 0.0335 (45.1\%) & 0.0266 (71.5\%) & 0.0224 (34.6\%) & 0.0948 (25.9\%) \\ \hline
\end{tabular}
\end{table}

\begin{table}[t]
\centering
\footnotesize
\caption{Server employs different strategies to collect global models.}
\label{tab:continuous_discrete}
\begin{tabular}{|l|c|c|c|c|}
\hline
Method & EO & DP & CAL & CON \\ \hline
FedAvg & 0.0611 & 0.0934 & 0.0343 & 0.1281 \\ \hline
Discrete & 0.0571 (6.5\%) & 0.0788 (15.6\%) & 0.0424 (-23.6\%) & 0.1196 (6.6\%) \\
Continuous & 0.0335 (45.1\%) & 0.0266 (71.5\%) & 0.0224 (34.6\%) & 0.0948 (25.9\%) \\ \hline
\end{tabular}
\end{table}

\begin{table}[t]
\centering
\scriptsize
\caption{Compare \alg with other debiasing methods.}
\label{tab:Regular}
\begin{tabular}{|l|c|c|c|c|}
\hline
Method & EO & DP & CAL & CON \\ \hline
FedAvg & 0.0611 & 0.0934 & 0.0343 & 0.1281 \\ \hline
Regular & 0.0561 (8.2\%) & 0.0874 (6.4\%) & 0.0289 (15.7\%) & 0.1253 (2.2\%) \\
\rowcolor{greyL}
\alg & 0.0335 (45.1\%) & 0.0266 (71.5\%) & 0.0224 (34.6\%) & 0.0948 (25.9\%) \\ \hline
\end{tabular}
\end{table}

\begin{table}[t]
\centering
\footnotesize
\caption{Server optimizes multiple fairness metrics simultaneously.}
\label{tab:moo_results}
\begin{tabular}{|l|c|c|c|c|}
\hline
    Method & EO & DP & CAL & CON \\
    \hline
    FedAvg & 0.0611 & 0.0934 & 0.0343 & 0.1281 \\
    \hline
    \rowcolor{greyL}
    \alg-Multi & 0.0363 (40.6\%) & 0.0457 (51.1\%) & 0.0223 (35.0\%) & 0.1003 (21.7\%) \\
    \hline
\end{tabular}
\end{table}

\begin{table}[t]
    \centering
      \footnotesize
    \caption{Storage cost of \alg.}
    \label{tab:storage_cost}
    \begin{tabular}{|l|c|c|c|}
        \hline
        Dataset & Global model (MB) & Synthetic dataset (MB) & Total (MB) \\
        \hline
        Income-Sex & 0.37 & 1.19 & 1.56 \\
        Employment-Sex & 0.37 & 1.19 & 1.56 \\
        Health-Sex & 0.37 & 1.19 & 1.56 \\
        Income-Race & 0.37 & 1.19 & 1.56 \\
        MNIST & 34.92 & 0.78 & 35.70 \\
        CIFAR-10 & 446.95 & 3.07 & 450.02 \\
        \hline
    \end{tabular}
\end{table}

\begin{figure*}[!t]
    \centering
    \includegraphics[width=\textwidth]{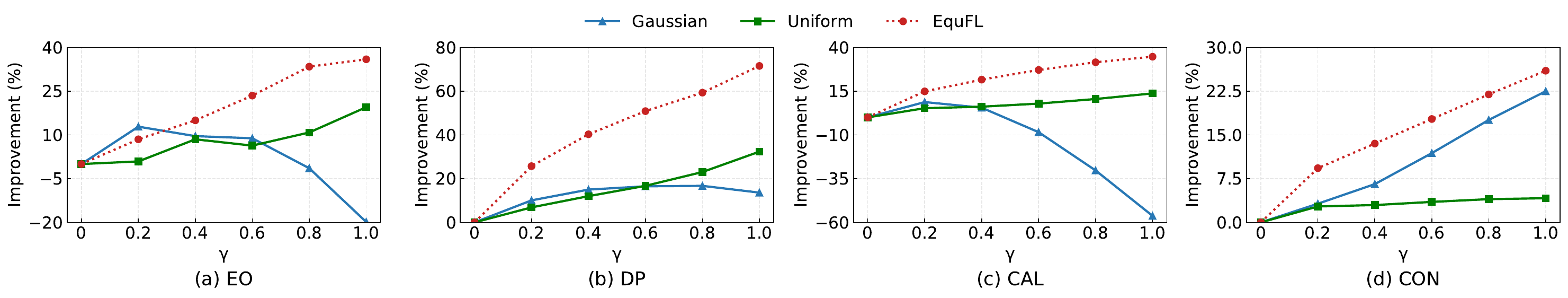}
    \caption{Impact of $\gamma$.}
    \label{fig:gamma}
\end{figure*}

\begin{figure*}[!t]
    \centering
    \includegraphics[width=\textwidth]{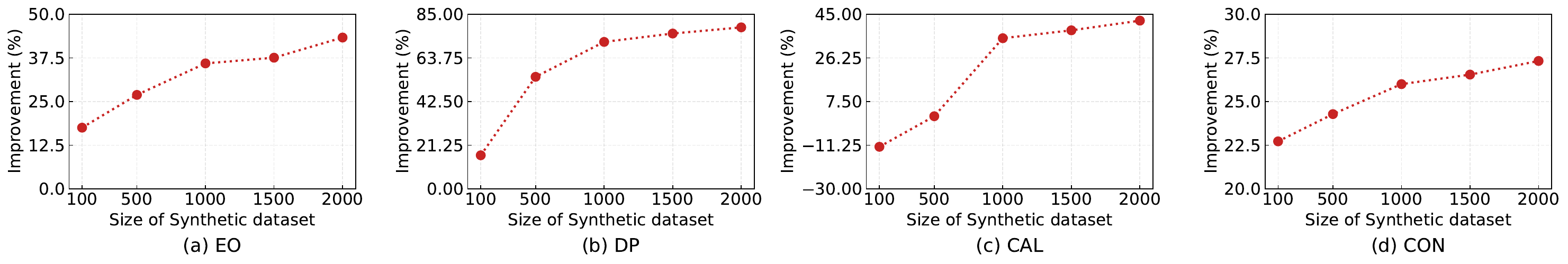}
    \caption{Impact of size of Synthetic dataset.}
    \label{fig:datasize}
\end{figure*}

\begin{figure*}[!t]
    \centering
    \includegraphics[width=\textwidth]{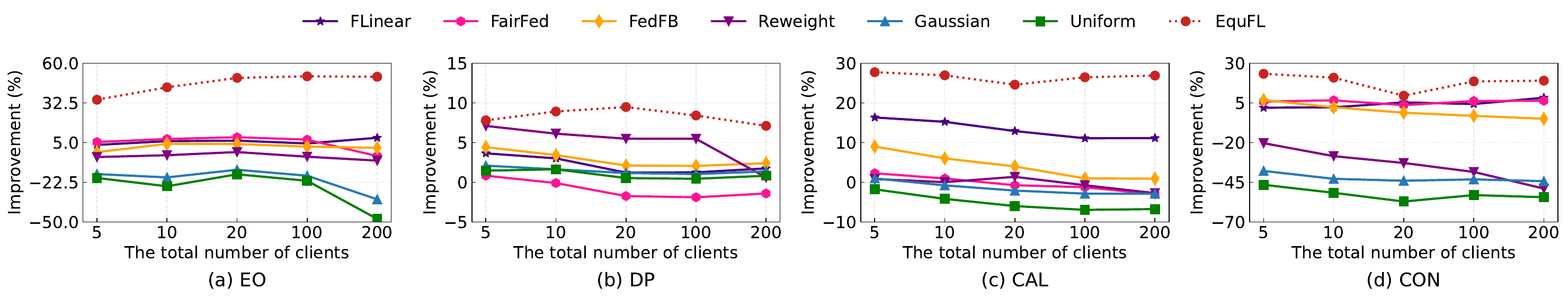}
    \caption{Impact of the total number of clients.}
    \label{fig:clientnum}
\end{figure*}

\begin{table}[t]
  \centering
  \footnotesize
  \caption{Performance of our \alg when clients add noise to their gradients before uploading them. $\sigma = 0$ indicates that no noise is added.}
  \label{tab:dp_noise_levels}
  \begin{tabular}{|l|c|c|c|c|}
    \hline
     Method & EO  & DP  & CAL  & CON \\
        \hline
     FedAvg ($\sigma{=}0$) & 0.0611 & 0.0934 & 0.0343 & 0.1281 \\ \hline
    \alg ($\sigma{=}0$) & 0.0335  & 0.0226 & 0.0224 & 0.0948 \\ 
    \alg ($\sigma{=}0.1$) &  0.0389&  0.0400&  0.0231 &  0.1036 \\ 
    \alg ($\sigma{=}0.2$) & 0.0407  & 0.0422  & 0.0251  & 0.1066\\ 
    \alg ($\sigma{=}0.3$) & 0.0456  & 0.0520  & 0.0305  & 0.1182 \\ \hline
  \end{tabular}
\end{table}

\begin{table}[t]
  \centering
  \footnotesize
  \caption{Test accuracy of the final global model learned by FedAvg when clients add noise to their gradients before uploading them. $\sigma = 0$ indicates that no noise is added. Note that the test accuracy of FedAvg remains the same across the ``EO'', ``DP'', ``CAL'', and ``CON'' metrics.}
  \label{tab:fedavg_dp_noise}
  \begin{tabular}{|l|c|}
    \hline
     Method & Test accuracy  \\
        \hline
    FedAvg ($\sigma{=}0$) & 0.7491  \\ 
    FedAvg ($\sigma{=}0.1$) &  0.7302  \\ 
    FedAvg ($\sigma{=}0.2$) &   0.7143\\ 
    FedAvg ($\sigma{=}0.3$) & 0.6960 \\ \hline
  \end{tabular}
\end{table}

%% file: mainfile.bbl
\begin{thebibliography}{10}
\providecommand{\url}[1]{#1}
\csname url@samestyle\endcsname
\providecommand{\newblock}{\relax}
\providecommand{\bibinfo}[2]{#2}
\providecommand{\BIBentrySTDinterwordspacing}{\spaceskip=0pt\relax}
\providecommand{\BIBentryALTinterwordstretchfactor}{4}
\providecommand{\BIBentryALTinterwordspacing}{\spaceskip=\fontdimen2\font plus
\BIBentryALTinterwordstretchfactor\fontdimen3\font minus
  \fontdimen4\font\relax}
\providecommand{\BIBforeignlanguage}[2]{{%
\expandafter\ifx\csname l@#1\endcsname\relax
\typeout{** WARNING: IEEEtran.bst: No hyphenation pattern has been}%
\typeout{** loaded for the language `#1'. Using the pattern for}%
\typeout{** the default language instead.}%
\else
\language=\csname l@#1\endcsname
\fi
#2}}
\providecommand{\BIBdecl}{\relax}
\BIBdecl

\bibitem{McMahan17}
H.~B. McMahan, E.~Moore, D.~Ramage, S.~Hampson, and B.~A. y~Arcas,
  ``Communication-efficient learning of deep networks from decentralized
  data,'' in \emph{AISTATS}, 2017.

\bibitem{chang2023bias}
H.~Chang and R.~Shokri, ``Bias propagation in federated learning,'' in
  \emph{ICLR}, 2023.

\bibitem{chen2023privacy}
H.~Chen, T.~Zhu, T.~Zhang, W.~Zhou, and P.~S. Yu, ``Privacy and fairness in
  federated learning: on the perspective of tradeoff,'' in \emph{ACM Computing
  Surveys}, 2023.

\bibitem{guo2023fedbr}
Y.~Guo, X.~Tang, and T.~Lin, ``Fedbr: Improving federated learning on
  heterogeneous data via local learning bias reduction,'' in \emph{ICML}, 2023.

\bibitem{li2021ditto}
T.~Li, S.~Hu, A.~Beirami, and V.~Smith, ``Ditto: Fair and robust federated
  learning through personalization,'' in \emph{ICML}, 2021.

\bibitem{li2019fair}
T.~Li, M.~Sanjabi, A.~Beirami, and V.~Smith, ``Fair resource allocation in
  federated learning,'' in \emph{ICLR}, 2020.

\bibitem{ezzeldin2023fairfed}
Y.~H. Ezzeldin, S.~Yan, C.~He, E.~Ferrara, and A.~S. Avestimehr, ``Fairfed:
  Enabling group fairness in federated learning,'' in \emph{AAAI}, 2023.

\bibitem{fan2022improving}
Z.~Fan, H.~Fang, Z.~Zhou, J.~Pei, M.~P. Friedlander, C.~Liu, and Y.~Zhang,
  ``Improving fairness for data valuation in horizontal federated learning,''
  in \emph{ICDE}, 2022.

\bibitem{he2020geometric}
Y.~He, K.~Burghardt, and K.~Lerman, ``A geometric solution to fair
  representations,'' in \emph{AAAI/ACM Conference on AI, Ethics, and Society},
  2020.

\bibitem{kamiran2012data}
F.~Kamiran and T.~Calders, ``Data preprocessing techniques for classification
  without discrimination,'' in \emph{Knowledge and information systems}, 2012.

\bibitem{roh2020fairbatch}
Y.~Roh, K.~Lee, S.~E. Whang, and C.~Suh, ``Fairbatch: Batch selection for model
  fairness,'' in \emph{ICLR}, 2021.

\bibitem{zeng2021improving}
Y.~Zeng, H.~Chen, and K.~Lee, ``Improving fairness via federated learning,''
  \emph{arXiv preprint arXiv:2110.15545}, 2021.

\bibitem{cao2020fltrust}
X.~Cao, M.~Fang, J.~Liu, and N.~Z. Gong, ``Fltrust: Byzantine-robust federated
  learning via trust bootstrapping,'' in \emph{NDSS}, 2021.

\bibitem{park2021sageflow}
J.~Park, D.-J. Han, M.~Choi, and J.~Moon, ``Sageflow: Robust federated learning
  against both stragglers and adversaries,'' in \emph{NeurIPS}, 2021.

\bibitem{wang2022flare}
N.~Wang, Y.~Xiao, Y.~Chen, Y.~Hu, W.~Lou, and Y.~T. Hou, ``Flare: defending
  federated learning against model poisoning attacks via latent space
  representations,'' in \emph{ASIACCS}, 2022.

\bibitem{xie2019zeno}
C.~Xie, S.~Koyejo, and I.~Gupta, ``Zeno: Distributed stochastic gradient
  descent with suspicion-based fault-tolerance,'' in \emph{ICML}, 2019.

\bibitem{dwork2012fairness}
C.~Dwork, M.~Hardt, T.~Pitassi, O.~Reingold, and R.~Zemel, ``Fairness through
  awareness,'' in \emph{ITCS}, 2012.

\bibitem{fang2022fairroad}
M.~Fang, J.~Liu, M.~Momma, and Y.~Sun, ``Fairroad: Achieving fairness for
  recommender systems with optimized antidote data,'' in \emph{SACMAT}, 2022.

\bibitem{hardt2016equality}
M.~Hardt, E.~Price, and N.~Srebro, ``Equality of opportunity in supervised
  learning,'' in \emph{NeurIPS}, 2016.

\bibitem{pleiss2017fairness}
G.~Pleiss, M.~Raghavan, F.~Wu, J.~Kleinberg, and K.~Q. Weinberger, ``On
  fairness and calibration,'' in \emph{NeurIPS}, 2017.

\bibitem{zemel2013learning}
R.~Zemel, Y.~Wu, K.~Swersky, T.~Pitassi, and C.~Dwork, ``Learning fair
  representations,'' in \emph{ICML}, 2013.

\bibitem{mohri2019agnostic}
M.~Mohri, G.~Sivek, and A.~T. Suresh, ``Agnostic federated learning,'' in
  \emph{ICML}, 2019.

\bibitem{shi2023towards}
Y.~Shi, H.~Yu, and C.~Leung, ``Towards fairness-aware federated learning,'' in
  \emph{IEEE Transactions on Neural Networks and Learning Systems}, 2023.

\bibitem{xu2023bias}
Y.-Y. Xu, C.-S. Lin, and Y.-C.~F. Wang, ``Bias-eliminating augmentation
  learning for debiased federated learning,'' in \emph{CVPR}, 2023.

\bibitem{zhang2021unified}
F.~Zhang, K.~Kuang, Y.~Liu, L.~Chen, C.~Wu, F.~Wu, J.~Lu, Y.~Shao, and J.~Xiao,
  ``Unified group fairness on federated learning,'' \emph{arXiv preprint
  arXiv:2111.04986}, 2021.

\bibitem{zhang2024eliminating}
J.~Zhang, Y.~Hua, J.~Cao, H.~Wang, T.~Song, Z.~Xue, R.~Ma, and H.~Guan,
  ``Eliminating domain bias for federated learning in representation space,''
  in \emph{NeurIPS}, 2024.

\bibitem{fang2020local}
M.~Fang, X.~Cao, J.~Jia, and N.~Gong, ``Local model poisoning attacks to
  byzantine-robust federated learning,'' in \emph{USENIX Security Symposium},
  2020.

\bibitem{shejwalkar2022back}
V.~Shejwalkar, A.~Houmansadr, P.~Kairouz, and D.~Ramage, ``Back to the drawing
  board: A critical evaluation of poisoning attacks on production federated
  learning,'' in \emph{IEEE Symposium on Security and Privacy}, 2022.

\bibitem{fang2022aflguard}
M.~Fang, J.~Liu, N.~Z. Gong, and E.~S. Bentley, ``Aflguard: Byzantine-robust
  asynchronous federated learning,'' in \emph{ACSAC}, 2022.

\bibitem{fang2024byzantine}
M.~Fang, Z.~Zhang, P.~Khanduri, J.~Liu, S.~Lu, Y.~Liu, N.~Gong \emph{et~al.},
  ``Byzantine-robust decentralized federated learning,'' in \emph{CCS}, 2024.

\bibitem{fang2025we}
M.~Fang, S.~Nabavirazavi, Z.~Liu, W.~Sun, S.~S. Iyengar, and H.~Yang, ``Do we
  really need to design new byzantine-robust aggregation rules?'' in
  \emph{NDSS}, 2025.

\bibitem{wang2025poisoning}
W.~Wang, Q.~Ma, Z.~Zhang, Y.~Liu, Z.~Liu, and M.~Fang, ``Poisoning attacks and
  defenses to federated unlearning,'' in \emph{Companion Proceedings of the ACM
  on Web Conference 2025}, 2025.

\bibitem{fang2025byzantine}
M.~Fang, Z.~Liu, X.~Zhao, and J.~Liu, ``Byzantine-robust federated learning
  over ring-all-reduce distributed computing,'' in \emph{Companion Proceedings
  of the ACM on Web Conference 2025}, 2025.

\bibitem{zhang2024poisoning}
Z.~Zhang, M.~Fang, J.~Huang, and Y.~Liu, ``Poisoning attacks on federated
  learning-based wireless traffic prediction,'' in \emph{IFIP/IEEE Networking
  Conference}, 2024.

\bibitem{fang2025provably}
M.~Fang, X.~Wang, and N.~Z. Gong, ``Provably robust federated reinforcement
  learning,'' in \emph{The Web Conference}, 2025.

\bibitem{cazenavette2022dataset}
G.~Cazenavette, T.~Wang, A.~Torralba, A.~A. Efros, and J.-Y. Zhu, ``Dataset
  distillation by matching training trajectories,'' in \emph{CVPR}, 2022.

\bibitem{kim2022dataset}
J.-H. Kim, J.~Kim, S.~J. Oh, S.~Yun, H.~Song, J.~Jeong, J.-W. Ha, and H.~O.
  Song, ``Dataset condensation via efficient synthetic-data parameterization,''
  in \emph{ICML}, 2022.

\bibitem{shen2022optimising}
A.~Shen, X.~Han, T.~Cohn, T.~Baldwin, and L.~Frermann, ``Optimising equal
  opportunity fairness in model training,'' in \emph{NAACL}, 2022.

\bibitem{paszke2019pytorch}
A.~Paszke, S.~Gross, F.~Massa, A.~Lerer, J.~Bradbury, G.~Chanan, T.~Killeen,
  Z.~Lin, N.~Gimelshein, L.~Antiga \emph{et~al.}, ``Pytorch: An imperative
  style, high-performance deep learning library,'' in \emph{NeurIPS}, 2019.

\bibitem{abadi2016tensorflow}
M.~Abadi, P.~Barham, J.~Chen, Z.~Chen, A.~Davis, J.~Dean, M.~Devin,
  S.~Ghemawat, G.~Irving, M.~Isard \emph{et~al.}, ``Tensorflow: a system for
  large-scale machine learning,'' in \emph{OSDI}, 2016.

\bibitem{ChenPOMACS17}
Y.~Chen, L.~Su, and J.~Xu, ``Distributed statistical machine learning in
  adversarial settings: Byzantine gradient descent,'' in \emph{POMACS}, 2017.

\bibitem{chu2022securing}
T.~Chu, A.~Garcia-Recuero, C.~Iordanou, G.~Smaragdakis, and N.~Laoutaris,
  ``Securing federated sensitive topic classification against poisoning
  attacks,'' in \emph{NDSS}, 2023.

\bibitem{karimireddy2020byzantine}
S.~P. Karimireddy, L.~He, and M.~Jaggi, ``Byzantine-robust learning on
  heterogeneous datasets via bucketing,'' in \emph{ICLR}, 2022.

\bibitem{li2019convergence}
X.~Li, K.~Huang, W.~Yang, S.~Wang, and Z.~Zhang, ``On the convergence of fedavg
  on non-iid data,'' in \emph{ICLR}, 2020.

\bibitem{yin2018byzantine}
D.~Yin, Y.~Chen, R.~Kannan, and P.~Bartlett, ``Byzantine-robust distributed
  learning: Towards optimal statistical rates,'' in \emph{ICML}, 2018.

\bibitem{liu2015deep}
Z.~Liu, P.~Luo, X.~Wang, and X.~Tang, ``Deep learning face attributes in the
  wild,'' in \emph{ICCV}, 2015.

\bibitem{lecun2010mnist}
Y.~LeCun, C.~Cortes, and C.~Burges, ``Mnist handwritten digit database,''
  \emph{Available: http://yann. lecun. com/exdb/mnist}, 1998.

\bibitem{krizhevsky2009learning}
A.~Krizhevsky, G.~Hinton \emph{et~al.}, ``Learning multiple layers of features
  from tiny images,'' 2009.

\bibitem{zhang2024lr}
Y.~Zhang and H.~Yu, ``Lr-xfl: logical reasoning-based explainable federated
  learning,'' in \emph{AAAI}, 2024.

\bibitem{zhang2025uncertainty}
------, ``Uncertainty-aware explainable federated learning,'' \emph{arXiv
  preprint arXiv:2503.05194}, 2025.

\bibitem{blanchard2017machine}
P.~Blanchard, E.~M. El~Mhamdi, R.~Guerraoui, and J.~Stainer, ``Machine learning
  with adversaries: Byzantine tolerant gradient descent,'' in \emph{NeurIPS},
  2017.

\bibitem{rieger2022deepsight}
P.~Rieger, T.~D. Nguyen, M.~Miettinen, and A.-R. Sadeghi, ``Deepsight:
  Mitigating backdoor attacks in federated learning through deep model
  inspection,'' in \emph{NDSS}, 2022.

\bibitem{binns2020apparent}
R.~Binns, ``On the apparent conflict between individual and group fairness,''
  in \emph{FAT}, 2020.

\bibitem{goethals2024beyond}
S.~Goethals, T.~Calders, and D.~Martens, ``Beyond accuracy-fairness: Stop
  evaluating bias mitigation methods solely on between-group metrics,''
  \emph{arXiv preprint arXiv:2401.13391}, 2024.

\bibitem{mashiat2022trade}
T.~Mashiat, X.~Gitiaux, H.~Rangwala, P.~Fowler, and S.~Das, ``Trade-offs
  between group fairness metrics in societal resource allocation,'' in
  \emph{FAccT}, 2022.

\bibitem{desideri2012multiple}
J.-A. D{\'e}sid{\'e}ri, ``Multiple-gradient descent algorithm (mgda) for
  multiobjective optimization,'' in \emph{Comptes Rendus Mathematique}, 2012.

\bibitem{abadi2016deep}
M.~Abadi, A.~Chu, I.~Goodfellow, H.~B. McMahan, I.~Mironov, K.~Talwar, and
  L.~Zhang, ``Deep learning with differential privacy,'' in \emph{CCS}, 2016.

\bibitem{he2016deep}
K.~He, X.~Zhang, S.~Ren, and J.~Sun, ``Deep residual learning for image
  recognition,'' in \emph{CVPR}, 2016.

\end{thebibliography}
